\documentclass[journal]{IEEEtran}

\usepackage{xcolor,soul,framed} 

\colorlet{shadecolor}{yellow}
\usepackage{color,soul}
\usepackage[pdftex]{graphicx}

\usepackage{cite}

\usepackage{graphicx,epstopdf}

\usepackage{amssymb,amsthm}
\usepackage{amsmath}
\usepackage{amsfonts}
\usepackage{mathrsfs}
\usepackage{tabularx} 
\usepackage[linesnumbered,boxed,ruled,commentsnumbered]{algorithm2e}

\usepackage{subfigure}
\usepackage{acronym}
\usepackage{todonotes}
\usepackage{color}
\usepackage{mathtools}
\usepackage{booktabs}
\usepackage{multirow}

\newtheorem*{remark}{Remark}
\newtheorem{theorem}{Theorem}
\newtheorem{assumption}{Assumption}
\newif\ifuseboldmathops
\newif\ifuseittextabbrevs

\ifuseittextabbrevs

\else

\fi

\ifuseboldmathops

\else

\fi

\ifuseboldmathops

\else

\fi

\ifuseboldmathops
	\newcommand{\Expect}{\mathop{\bf E{}}\nolimits}
	
\else
	\newcommand{\Expect}{\mathop{\mathbb{E}{}}\nolimits}
	
\fi

\ifuseboldmathops

\else

\fi

\newcommand{\argmin}{\mathop{\mathrm{argmin}}}

\newcommand{\norm}[1]{\lVert#1\rVert}

\newcommand{\calF}{\mathcal{F}}

\acrodef{mdp}[MDP]{Markov decision process}
\acrodef{dfa}[DFA]{deterministic finite-state automaton}
\acrodef{ltl}[LTL]{linear temporal logic}
\acrodef{ltlf}[LTL$(\calF)$]{quantitative linear temporal logic}
\acrodef{ag}[AG]{Assume-Guarantee}
\acrodef{ssp}[SSP]{Stochastic Shortest Path}
\acrodef{mcmc}[mcmc]{Monte Carlo Markov chain}

\theoremstyle{definition}
 \newtheorem{definition}{Definition}
 \newtheorem{example}{Example}

\newtheorem{lemma}{Lemma}

\acrodef{ltl}[LTL]{linear temporal logic formula}
\acrodef{mdp}[MDP]{Markov decision process}
\acrodef{smdp}[Semi-MDP]{Semi-Markov decision process}

\acrodef{scltl}[scLTL]{syntactically co-safe LTL}
\newcommand{\defeq}{\mathrel{\mathop:}=}

\newtheorem{prop}{Proposition}

\begin{document}
\title{A Multi-Agent Reinforcement Learning Approach For Safe and Efficient Behavior Planning Of Connected Autonomous Vehicles}
\author{Songyang Han, Shanglin Zhou, Jiangwei Wang, Lynn Pepin, Caiwen Ding,~\IEEEmembership{Member,~IEEE,} \\ Jie Fu,~\IEEEmembership{Member,~IEEE,} Fei Miao,~\IEEEmembership{Member,~IEEE}
\thanks{This work was supported by NSF 1849246, NSF 1952096, and NSF 2047354 grants. }
\thanks{S.~Han, S.~Zhou, L.~Pepin, C.~Ding, and F.~Miao are with the Department of Computer Science and Engineering, and J.~Wang is with the Department of Electrical and Computer Engineering, University of Connecticut, Storrs Mansfield, CT, USA 06268. }
\thanks{J.~Fu is with the Department of Electrical and Computer Engineering, University of Florida, Gainesville, FL, USA 32605.}
\thanks{Email: \{songyang.han, shanglin.zhou, jiangwei.wang, lynn.pepin, caiwen.ding, fei.miao\}@uconn.edu, fujie@ufl.edu.}}

\maketitle

\begin{abstract}
The recent advancements in wireless technology enable connected autonomous vehicles (CAVs) to gather information about their environment by vehicle-to-vehicle (V2V) communication. In this work, we design an information-sharing-based multi-agent reinforcement learning (MARL) framework for CAVs, to take advantage of the extra information when making decisions to improve traffic efficiency and safety. 
The safe actor-critic algorithm we propose has two new techniques: the truncated $\mathcal{Q}$-function and safe action mapping. The truncated $\mathcal{Q}$-function utilizes the shared information from neighboring CAVs such that the joint state and action spaces of the $\mathcal{Q}$-function do not grow in our algorithm for a large-scale CAV system. 
We prove the bound of the approximation error between the truncated-$\mathcal{Q}$ and global $Q$-functions. The safe action mapping provides a provable safety guarantee for both the training and execution based on control barrier functions.
Using the CARLA simulator for experiments, we show that our approach can improve the CAV system's efficiency in terms of average velocity and comfort under different CAV ratios and different traffic densities. We also show that our approach avoids the execution of unsafe actions and always maintains a safe distance from other vehicles. We construct an obstacle-at-corner scenario to show that the shared vision can help CAVs to observe obstacles earlier and take action to avoid traffic jams. 
\end{abstract}


\section{Introduction}
\label{sec:introduction}
\IEEEPARstart{W}{ireless} communication technologies such as WiFi and 5G cellular networks help enable vehicle-to-vehicle (V2V) communication~\cite{martin2020low,mun2021secure}. The U.S. Department of Transportation (DOT) estimated that DSRC (dedicated short-range communications)-based V2V communication could address up to 82\% of crashes in the U.S. every year~\cite{DSRC_standard,won2021platooning}. Sharing basic safety messages (BSMs) benefits the coordination of connected autonomous vehicles (CAVs) at intersections and lane-merging scenarios~\cite{Coordinate_CAV, el2022novel}.
However, when CAVs get extra environment knowledge via V2V communication, how to make prudent decisions to improve traffic efficiency, and whether sharing vision information can bring benefits are still unsolved challenges. Hence, we show how CAVs can take advantage of information sharing to make better driving decisions while meeting safety requirements and improving traffic efficiency. 

In this work, We consider utilizing V2V communication among CAVs to make better behavior planning and control decisions, such as lane-changing and lane-keeping. This problem is not well-studied. Most existing CAV frameworks assume that the lane-changing or keeping decision has already been made, such as platooning~\cite{won2021platooning}, adaptive cruise control (ACC)~\cite{farivar2021security}, and cooperative adaptive cruise control (CACC)~\cite{silgu2021combined}. 
Since the driving environment is quite complicated given the large state space for autonomous vehicles, reinforcement learning has shown advantages in literature for a single autonomous vehicle's decision making~\cite{chen2020conditional,chen2021interpretable,kiran2021deep,el2022novel,fu2022hybrid,zhu2022operational}. However, how to use reinforcement learning for multiple vehicles' decision making has not been well-studied yet~\cite{kiran2021deep}. It is natural to consider multi-agent reinforcement learning (MARL) for CAVs' decision making. Nevertheless, since a centralized critic with access to the global state and the global action is required for the MARL~\cite{lowe2017multi, rashid2018qmix}, there are three challenges, in reality, preventing us from using the existing MARL algorithms in the literature:
\begin{itemize}
    \item It is difficult for each CAV to get the global state and the global action considering the communication overheads.
    \item The joint state and action spaces of the critic (Q-function) grow combinatorially with the total number of CAVs. The computational overheads become burdensome when the total number of CAVs becomes very large.
    \item The action used for both training and execution may be unsafe for the safety-critical CAV system~\cite{muhammad2020deep}.
\end{itemize}

To tackle these challenges, we design a new algorithm, called the safe actor-critic algorithm, with two new techniques: truncated $\mathcal{Q}$-function and safe action mapping.
To the best of our knowledge, this is the first attempt to design a safe MARL algorithm to solve the behavior planning challenges for CAVs considering the dynamic control process. We design a truncated $\mathcal{Q}$-function with neighboring vehicles' states and actions to approximate the global action value function with global states and actions in the MARL. 
We design a safe action mapping algorithm based on control barrier functions (CBFs) to make sure that the actions trained and executed by our proposed MARL algorithm are safe. 
In experiments, we show that our approach increases traffic efficiency and guarantees safety.

In summary, the main contributions of this work are:
\begin{itemize}
    \item We propose a novel safe and efficient actor-critic algorithm for behavior planning of CAVs based on two new techniques: 1) \textit{Truncated $\mathcal{Q}$-function} (for the first two challenges above): Each vehicle learns a truncated $\mathcal{Q}$-function as a critic that only needs the states and actions from neighboring vehicles. The joint state and action spaces of the truncated $\mathcal{Q}$-function do not grow in a large-scale CAV system. 2) \textit{Safe action mapping} (for the third challenge above): We map any action in the action space to the safe action set so that the training and execution have provable safety guarantees.
    \item To support the learning process of the truncated $\mathcal{Q}$-function, we propose a weight-pruned convolutional neural network (CNN) technique to guarantee the images from the camera and point clouds from LIDAR can be processed fast enough such that the vision information is always available for the learning of the truncated $\mathcal{Q}$-function.
    \item We validate our algorithms in the CARLA simulator~\cite{Dosovitskiy17} that can simulate complicated mixed traffic environments including both autonomous and human-driven vehicles. The experiments show that the safe actor-critic algorithm can improve traffic efficiency with safety guarantees. We also validate our MARL algorithm in challenging driving scenarios like obstacle-at-corner, and the shared vision with our algorithm helps vehicles to avoid traffic jams.
\end{itemize}

The rest of this paper is organized as follows. We introduce the related work in Section~\ref{sec:related_work}. In Section~\ref{sec:problem}, we introduce the V2V communication setting, formulate the behavior planning and control problem, and give an overview of our novel solutions. In Section~\ref{sec:info_sharing} we propose the truncated $\mathcal{Q}$-function with provable approximation error. In Section~\ref{sec:learning_control}, we design a safe action mapping technique and propose a safe actor-critic algorithm for behavior planning. In Section~\ref{sec:experiment}, we show the experimental results with safety and efficiency improvements. Section~\ref{sec:conclusion} concludes the work.

\section{Related Work}
\label{sec:related_work}
\subsection{Deep Learning Framework For Autonomous Driving}

It is productive for autonomous vehicles to learn the steering angle and acceleration directly based on vision perception, such as end-to-end imitation learning~\cite{le2022survey}, and end-to-end reinforcement learning~\cite{cheng2019end,chen2020conditional,chen2021interpretable}.
However, end-to-end frameworks usually only show effectiveness in lane-keeping scenarios without lane-changing behaviors~\cite{bojarski2016end,cheng2019end,chen2020conditional,chen2021interpretable}. Moreover, the end-to-end frameworks do not have a provable safety guarantee. We also show in our experiments that vehicles have zigzag trajectories using end-to-end learning without explicitly deciding to change lanes but directly deciding steering angle and acceleration.
Therefore, the end-to-end framework does not fit our CAV problem.

When considering lane-changing behaviors, it is popular to separate the learning and control modules~\cite{shalev2016safe,aradi2020survey,he2021rule,zhang2022unified}. The learning module makes a high-level decision, such as ``go straight", ``go left"~\cite{aradi2020survey}, or ``yield to another vehicle"~\cite{shalev2016safe}. Then the control module implements the high-level decision with a safety guarantee~\cite{aradi2020survey}. Therefore, we design a discrete action space with continuous control inputs to execute the lane-keeping or lane-changing maneuver in our proposed MARL algorithm. Existing literature only considers the learning and control for a single autonomous vehicle, while in this work we solve the more challenging learning and control problem for a multi-vehicle system.
 
\subsection{Multi-agent Reinforcement Learning}
Existing multi-agent reinforcement learning (MARL) literature~\cite{zhang2019multi} has not fully solved the challenges for CAVs. How communication among agents will improve systems' safety or efficiency in policy learning has not been addressed. Recent advancements like multi-agent deep deterministic policy gradient (MADDPG)~\cite{lowe2017multi}, the attention mechanism~\cite{iqbal2019actor}, cooperative MARL~\cite{rashid2018qmix,sunehag2018value,yu2019distributed,rashid_nips20,NEURIPS2020_8977ecbb} and league training~\cite{vinyals2019grandmaster}, do not specify communication among agents or safety guarantees for the learned policy. A recent MARL work designs a scalable actor-critic algorithm for networked systems~\cite{qu2020scalable}, but the method cannot be applied to physical systems like CAVs or provide critical safety guarantees. Moreover, their localized policy only relies on the local state, while in our design the localized policy utilizes the information sharing capability of CAVs and the shared states from neighbors. Therefore, we consider a novel problem of CAVs' planning with information sharing and design a safe MARL algorithm with a scalable critic function and provable safety guarantees. 

\subsection{Safe Reinforcement Learning}
Safe reinforcement learning is an increasingly important research area for real-world safety-critical applications~\cite{bastani2018verifiable, thomas2021safe}. 
The existing safe RL methods for single-agent RL cannot be directly used to solve the challenges in MARL considered in this work. The Monte Carlo tree search method is used to search for safe actions in single-agent RL~\cite{mo2022safe}, but the search space grows combinatorially large without a formal safety guarantee in MARL. Control barrier functions (CBFs) have been applied to guarantee the safety of end-to-end learning for a single vehicle with continuous action space~\cite{cheng2019end}. However, we consider a more challenging multi-agent problem with discrete action and continuous control inputs.

The existing safe MARL methods either cannot solve the CAV challenges or interrupt the learning process of the MARL agents. Safe MARL methods mainly have two types: constrained MARL or shielding for exploration. In constrained MARL~\cite{lu2021decentralized}, agents maximize the total expected return while keeping the costs lower than designed bounds. However, the constraints in these methods cannot explicitly represent the safety requirement at every timestep of a physical dynamic system like CAVs, and safety is rarely guaranteed during learning in practice. The model predictive shielding (MPS) algorithm provably guarantees safety for any learned MARL policy~\cite{li2020robust, zhang2019mamps}. The basic idea is to use a backup controller to override the learned policy by dynamically checking whether the learned policy can maintain safety. 
However, this overriding interrupts the learning process. Our proposed algorithm maps any action in the action space to a safe action such that the RL agent can keep learning without being interrupted. Recent work proposes a joint learning method of learning the controller and the control barrier certificates with policy refinement~\cite{qin2021learning}. But they use supervised learning and they cannot guarantee the learned control barrier certificates work well for different datasets. In contrast, we design a safe action mapping algorithm to guarantee the safety of both training and execution of MARL considering vehicles' physical dynamics.

\section{Problem Description and Solution Overview}
\label{sec:problem}
We consider behavior planning (such as deciding when to change or keep lanes) and control of CAVs to improve traffic efficiency with safety guarantees. The driving environment is a multiple-lane freeway mixed with CAVs and human-driven vehicles, and our decision-making framework is designed for the CAVs. Human-driven vehicles are considered part of the environment and can be observed by the onboard sensors.

We assume that each CAV receives the following information from its neighboring CAVs by V2V communication:
\begin{itemize}
    \item Each neighboring CAV's position, velocity, and acceleration.
    \item Each neighboring CAV's vision information (detection results) captured by the onboard camera and LIDAR sensors and processed by a convolutional neural network. 
    \item Each neighboring CAV's action in the action set \{Change Left, Change Right, Keep Lane\}.
\end{itemize}
Sharing vehicle's states and actions has already been used in many CAV applications~\cite{Coordinate_CAV, eskandarian2019research, el2022novel, berlato2022smart}, our V2V communication demand is satisfied by the recent V2V communication advances~\cite{won2021platooning,mun2021secure}.




\begin{figure}[htb]
  \centering
  \includegraphics[width=3.4in]{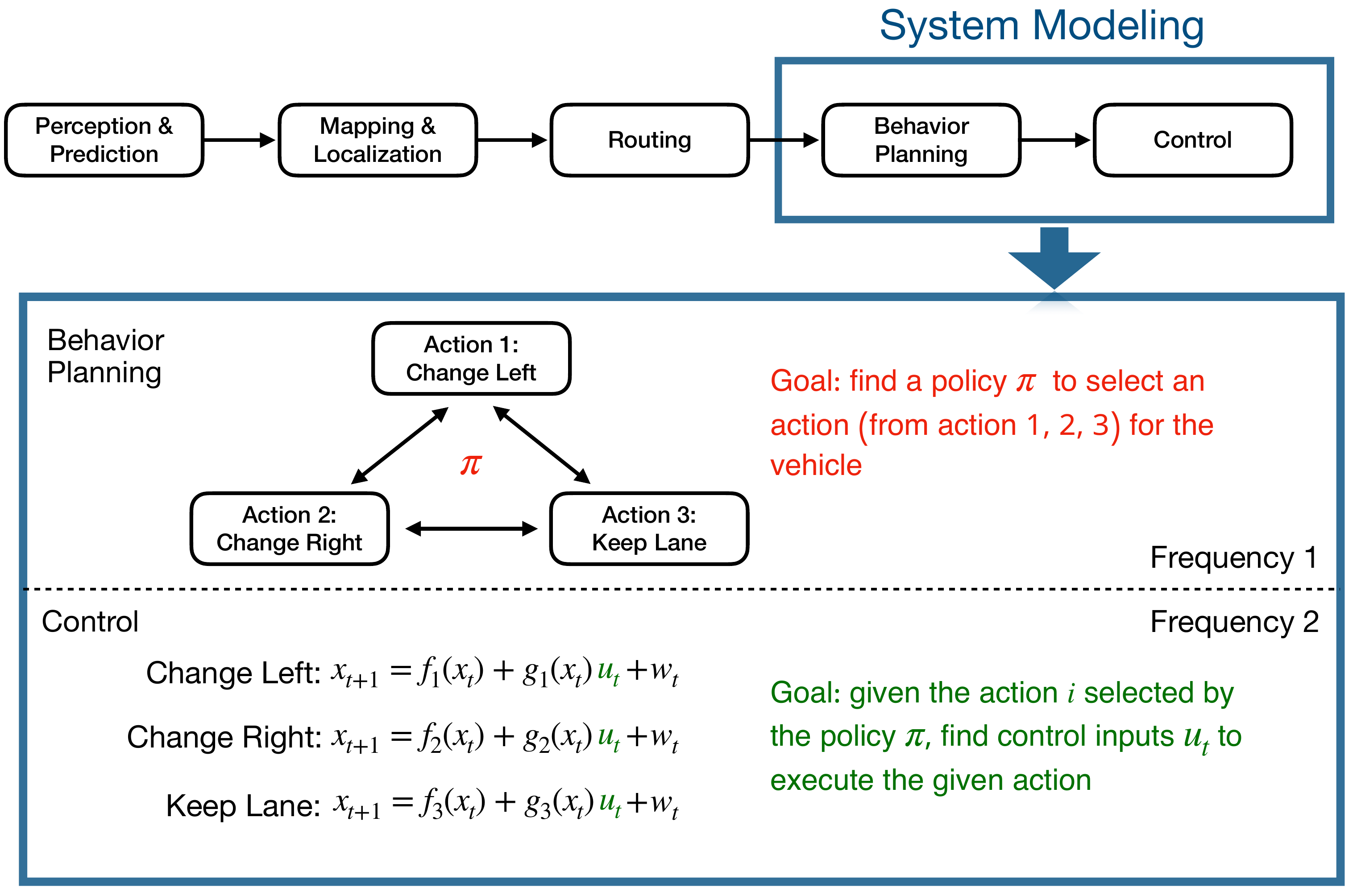}
  \caption{The system modeling of the behavior planning and control problem considered in this work. We design a novel safe behavior planning and control framework with decentralized training and decentralized execution to tackle the new challenges for CAVs.}\label{fig:system_modeling}
\end{figure}

One typical workflow for an autonomous vehicle includes perception, prediction, mapping and localization, routing, behavior planning, and control~\cite{aradi2020survey}. We focus on the last two modules: the behavior planning module to determine whether to change or keep lanes; the control module to control the steering angle and the acceleration. The system modeling is shown at the bottom of Fig.~\ref{fig:system_modeling}. In our framework, the behavior planning module selects one of the discrete actions. We aim to find a policy for each CAV to decide the action to improve traffic efficiency with safe control inputs.
For an autonomous vehicle, the behavior planning module and the control module work in different frequencies. One example is that the behavior planning module updates at 2Hz and the control module updates at 100Hz~\cite{notomista2020enhancing}. This is one reason why we
separate the behavior learning and control modules.

In summary, our goal is to design a behavior planning and control framework for CAVs to improve the transportation system's safety and efficiency. 

\subsection{Problem Formulation}
\label{sec:formulation}

In this section, we formulate the problem for the MARL-based behavior planning and the continuous space physical dynamic control module in Fig.~\ref{fig:system_modeling}.

\subsubsection{Behavior Planning}
We consider $n$ CAVs with an undirected graph $\mathcal{G} = (\mathcal{N}, \mathcal{E})$ as the communication connections. The nodes $\mathcal{N} = \{1, ..., n\}$ is the set of CAVs and $\mathcal{E} \subset \mathcal{N} \times \mathcal{N}$ is the edge set. Each vehicle $i$ is associated with an action $a^i \in \mathcal{A}^i$ and a state $s^i \in \mathcal{S}^i$. The action set $\mathcal{A}^i$ is \{Keep Lane (KL), Change Left (CL), and Change Right (CR)\}. The state $s^i$ of each vehicle is \{position, velocity, acceleration, vision information\}, where the vision information is captured by the onboard camera and LIDAR sensors and processed according to the weight-pruned CNN to be introduced in Section~\ref{sec:cnn_vision}. The global state is $ s = (s^1, ..., s^n) \in \mathcal{S} \defeq \mathcal{S}^1 \times \cdots \times \mathcal{S}^n$. The global action is $ a = (a^1, ..., a^n) \in \mathcal{A} \defeq \mathcal{A}^1 \times \cdots \times \mathcal{A}^n$.

Each vehicle has a stage-wise reward function $r^i(s^i, a^i)$ that depends on its local state and action. The global stage-wise reward is $r(s,a) = \frac{1}{n} \sum_{i=1}^n r^i(s^i, a^i)$. Each vehicle is associated with a localized policy $\pi^i(a^i| s^{N^i})$ given the joint states $s^{N^i}$ of its neighbors (including itself). We use $\pi(a | s)$ to denote the joint policy of $n$ CAVs.
The action-value function for the joint policy $\pi$ is
\begin{align}
    & Q(s, a) = \Expect_{a_k \sim \pi(a|s)} \left[ \sum_{k=0}^\infty \gamma^k r_{k+1}(s_k, a_k) | s_0 = s, a_0 = a \right] \nonumber \\
    &= \frac{1}{n} \sum_{i=1}^n \Expect_{a_k \sim \pi(a |s)} \left[ \sum_{k=0}^\infty \gamma^k r^i_{k+1}(s^i_k, a^i_k) | s_0 = s, a_0 = a \right] \nonumber \\
    & \defeq \frac{1}{n} \sum_{i=1}^n Q^i(s, a),
    \label{equ:q_function}
\end{align}
where $Q^i(s, a)$ is the action value function for each vehicle $i$ and $\gamma \in (0,1)$ is a discount factor. 

The objective is to find a policy $\pi$ to maximize the total expected return of $n$ CAVs:
\begin{equation}
    G = \Expect_{s_0 \sim p_0} \Expect_{a_k \sim \pi(a|s)} \left[ \sum_{k=0}^\infty \gamma^k r_{k+1}(s_k, a_k) \mid s_0 \right],
\end{equation}
where $p_0$ is the initial state distribution. 

\subsubsection{Control}
As shown in Fig.~\ref{fig:system_modeling}, we consider the physical dynamics for each vehicle $i$ in the control affine form:
\begin{equation}
    x_{t+1} = f(x_t) + g(x_t) u_t + w_t. \footnote{The controller uses the same design for each vehicle $i$. We drop the superscript $i$ in control states and control inputs when there is no confusion.}
    \label{equ:low_sys_equation_with_noise}
\end{equation}
with $f$ and $g$ locally Lipschitz~\footnote{A function $f:\mathbb{R}^{n_x} \rightarrow \mathbb{R}$ is called Lipschitz continuous over $\mathcal{X}$ if and only if there exists a constant $K \geq 0$ such that for any $x_1, x_2 \in \mathcal{X}$, 
\begin{equation}
    \left| f(x_1)- f(x_2) \right| \leq K \left| x_1 - x_2 \right|. \nonumber
\end{equation}
A function $f:\mathbb{R}^{n_x} \rightarrow \mathbb{R}$ is called locally Lipschitz continuous if for every $x \in \mathcal{X}$ there exists a neighborhood $\mathcal{X}_x$ of $x$ such that $f$ is Lipschitz continuous over $\mathcal{X}_x$.}, $x \in \mathcal{X} \subset \mathbb{R}^{n_x}$ is the control state and $u \in \mathcal{U} \subset \mathbb{R}^{m_u}$ is the control input, $w$ is the bounded noise that satisfies $\norm{w_t} \leq W$ for all $t \geq 0$.

We want to find control inputs such that the control state $x_t$ is always in a safe set $\mathcal{C} \subseteq \mathcal{X}$. In other words, the ego vehicle stays within a given range or a bounding box of its planned trajectory and does not collide with other vehicles or obstacles.



Note that we use $t$ and $k$ to distinguish between the controller timestep and reinforcement learning timestep. We want to find a policy $\pi(a_k|s_k)$ to maximize the total expected return $G$; meanwhile, there should exist control inputs $u_t \in \mathcal{U}$ for each vehicle to execute the action and guarantee that $x_t \in \mathcal{C}$ for all $t \geq 0$ given the initial state $x_0 \in \mathcal{C}$.


\subsection{Our Novel Solution Overview}

\begin{figure}[h]
  \centering
  \includegraphics[width=3.4in]{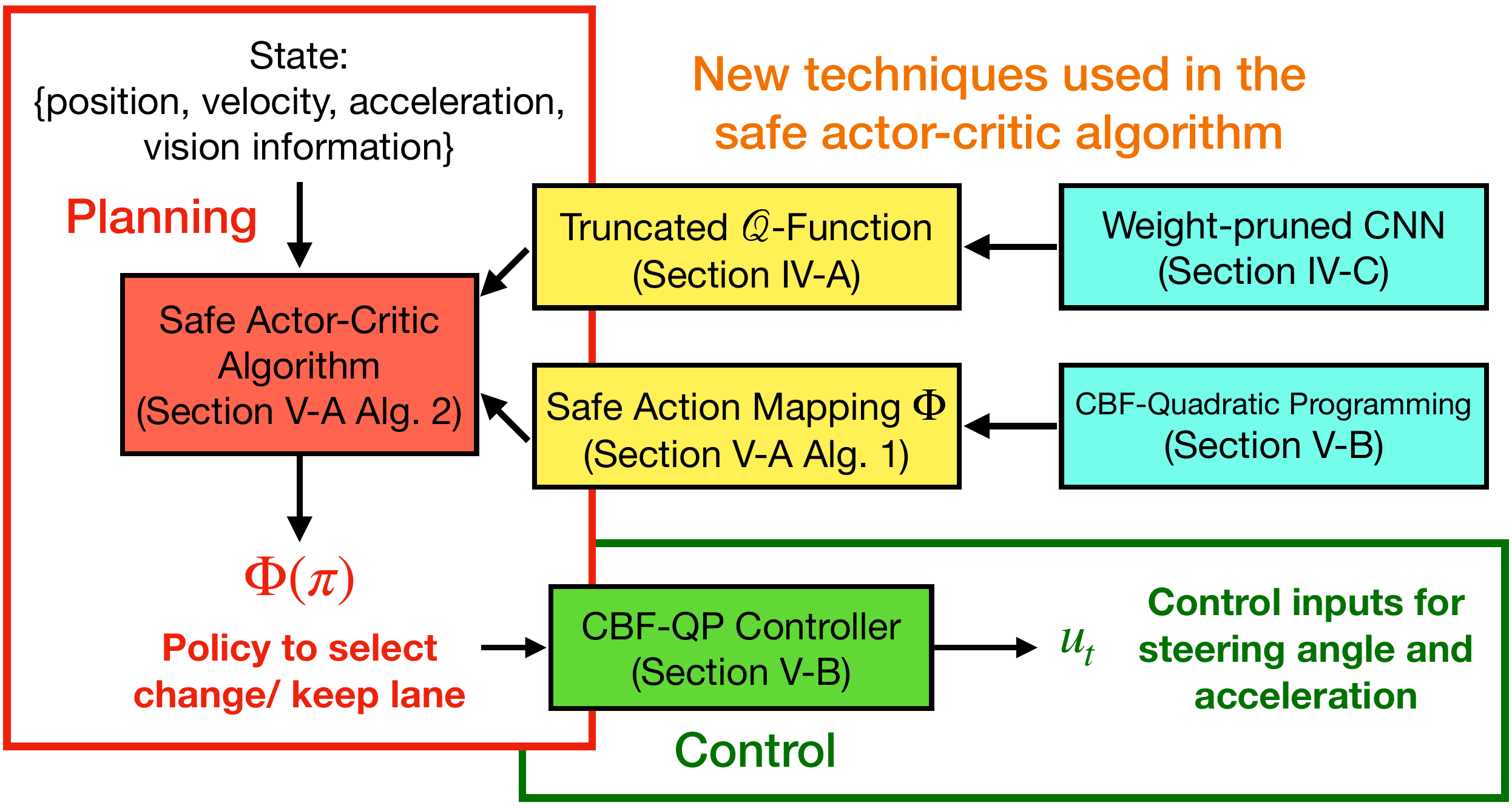}
\vspace{-5pt}
  \caption{Our proposed learning and control framework for the behavior planning and control problem. We design a safe actor-critic MARL algorithm to learn a policy to select actions. We use two new techniques in our algorithm: truncated $\mathcal{Q}$-function and safe action mapping. We also introduce a CBF-QP controller to generate control inputs for steering angle and acceleration with provable safety guarantees.}
  \label{fig:learning_control_design}
\end{figure}

In this work, we design a novel behavior learning and control framework with \textit{a safe actor-critic algorithm for decentralized training and decentralized execution} as shown in Fig.~\ref{fig:learning_control_design}. We assume that all CAVs share their states (including weight-pruned CNN processed vision information) and actions with others. Each CAV uses minibatch gradient descent to learn the truncated $\mathcal{Q}$-function as a critic to be introduced in Section~\ref{sec:info_sharing}. Then each CAV learns a localized policy $\Phi(\pi(a^i|s^{N^i}))$ based on the learned critic with a safe action mapping $\Phi$ to be introduced in Section~\ref{sec:learning_control}. The details of this safe actor-critic algorithm will be introduced in Section~\ref{sec:safe_algorithm}.

\section{Truncated $\mathcal{Q}$-Function}
\label{sec:info_sharing}




In this section, we introduce the design of truncated action value function $\mathcal{Q}$ to solve the first two challenges mentioned in the introduction such that the training process utilizes the information sharing capability of neighboring CAVs instead of relying on the global states and actions of all agents. Moreover, the joint state and action spaces of the truncated $\mathcal{Q}$-function do not grow with the total number of CAVs.

We first define a truncated $\mathcal{Q}$-function in subsection $A$ to approximate the global $Q$-function and prove that the approximation error is bounded in Theorem~\ref{lemma:exponential_decay} and Theorem~\ref{lemma:truncatedQ}. This is one main result of this work. Then, we introduce the truncated $\mathcal{Q}$-network design. At last, we introduce how to overcome the challenges of sharing vision information (required in the input of truncated $\mathcal{Q}$-function) using the weight-pruned CNN.

\subsection{Truncated $\mathcal{Q}$-function Approximation}
\label{sec:truncated_Q_approximation}

To tackle the new challenges for the CAVs, we introduce how to use the truncated $\mathcal{Q}$-function to approximate the global $Q$-function. The key idea is that the further the vehicles are away from the ego vehicle~\footnote{We refer to the vehicle we focus on (vehicle $i$) as the ego vehicle.}, the less impact they will have on the ego vehicle. To illustrate this idea, we first introduce one assumption:
\begin{assumption}
\label{aspt:independence}
At each time instant $k$, each vehicle $i$'s next state $s^i_{k+1}$ is independently generated and only depends on its neighbors, hence,
\begin{equation}
\label{equ:state_transition}
    P(s_{k+1} | s_k, a_k) = \prod_{i=1}^{n} P(s^i_{k+1} | s^{N^i} _k, a^{N^i} _k),
\end{equation}
where $N^i$ means the 1-hop neighborhood of vehicle $i$ including itself, $s^{N^i}$ is the joint states of $N^i$, and $a^{N^i}$ is the joint actions of $N^i$. 
\end{assumption}
This assumption follows the idea that vehicles far away from the ego vehicle do not have a direct influence on the ego vehicle. Then we introduce the definition of the exponential decay property. We use $N_\kappa^i$ to denote the $\kappa$-hop neighborhood of vehicle $i$ for $\kappa \geq 1$ (including all vehicles whose graph distance to $i$ is less than or equal to $\kappa$). We use $N_\kappa^{-i}$ to denote the vehicles that are outside of vehicle $i$'s $\kappa$-hop neighborhood. Then the global state $s$ can be divided into two parts $(s^{N^i_\kappa}, s^{N^{-i}_\kappa})$. Similarly, we divide the global action $a$ into $(a^{N^i_\kappa}, a^{N^{-i}_\kappa})$. The exponential decay property is defined for the global $Q^i$-function as follows.
\begin{definition}[Exponential decay property~\cite{qu2020scalable}]
\label{def:exponential_decay}
The $(c,\rho)$-exponential decay property holds for the global $Q^i$-function in \eqref{equ:q_function} if there exists some $c > 0$ and $ 0<\rho <1$ such that for any $i\in \mathcal{N}, \allowbreak \forall s^{N^i_\kappa} \in \mathcal{S}^{N^i_\kappa}, \allowbreak \forall s^{N^{-i}_\kappa} \in \mathcal{S}^{N^{-i}_\kappa}, \allowbreak \forall s'^{N^{-i}_\kappa} \in \mathcal{S}^{N^{-i}_\kappa}, \allowbreak \forall a^{N^i_\kappa} \in \mathcal{A}^{N^i_\kappa}, \allowbreak \forall a^{N^{-i}_\kappa} \in \mathcal{A}^{N^{-i}_\kappa}, \allowbreak \forall a'^{N^{-i}_\kappa} \in \mathcal{A}^{N^{-i}_\kappa}$, it holds that $|Q^i(s^{N^i_\kappa}, s^{N^{-i}_\kappa}, a^{N^i_\kappa}, a^{N^{-i}_\kappa})- Q^i(s^{N^i_\kappa}, s'^{N^{-i}_\kappa}, a^{N^i_\kappa}, a'^{N^{-i}_\kappa})| \le c \rho^{\kappa}$.
\end{definition}

This property shows that the impact of the states and the actions of the $\kappa$-hop neighborhood on the global $Q^i$-function decreases exponentially as $\kappa$ increases. The following Theorem~\ref{lemma:exponential_decay} shows that the exponential decay property holds for the global $Q^i$-function in our CAV MARL problem. The proof of this theorem is postponed to Appendix~\ref{sec:exponential_decay}. This property is utilized for CAVs to get rid of the dependence on the global states and actions.


\begin{theorem}
\label{lemma:exponential_decay}
If for all $i \in \mathcal{N}$, the reward $r^i$ is upper bounded by $\bar{r}$, then the $(\frac{\bar{r}}{1- \gamma}, \sqrt{\gamma})$-exponential decay property holds for $Q^i$ in \eqref{equ:q_function} under Assumption~\ref{aspt:independence}, where $\gamma$ is the discount factor.
\end{theorem}

\begin{proof}
See Appendix~\ref{sec:exponential_decay}.
\end{proof}

Theorem~\ref{lemma:exponential_decay} shows that with the exponential decay property, the states and actions of vehicles far away have limited contribution to estimating the global $Q^i$-function for the agent $i$. Based on this property, we define a truncated $\mathcal{Q}$-function as follows to approximate the global $Q^i$-function.
\begin{definition}[Truncated $\mathcal{Q}$-function]
\label{def:truncatedQ}
The truncated $\mathcal{Q}$-function for each vehicle $i$ only takes in the states and actions of the $\kappa$-hop neighborhood of vehicle $i$ as follows:
\begin{equation}
\mathcal{Q}^i(s^{N^i_\kappa}, a^{N^i_\kappa}) \defeq Q^i(s^{N^i_\kappa}, \mathbf{0}, a^{N^i_\kappa}, \mathbf{0}),
\end{equation}
where $s^{N^{-i}_\kappa}$ and $a^{N^{-i}_\kappa}$ are selected as $\mathbf{0}$~\footnote{We assume that $\mathbf{0}$ represents a valid state or action in $\mathcal{S}^{N^{-i}_\kappa}$ and $\mathcal{A}^{N^{-i}_\kappa}$.}.
\end{definition}



The benefit of using the truncated $\mathcal{Q}$-function to approximate the global $Q$-function is that the truncated $\mathcal{Q}$-function only needs the states and actions of the $\kappa$-hop neighborhood of vehicle $i$ instead of the global states and global actions. The joint state and action spaces of the truncated $\mathcal{Q}$-function do not grow with the total number of CAVs. We have the following theorem to give a bound of the approximation error by using the truncated $\mathcal{Q}$-function.

\begin{theorem}
\label{lemma:truncatedQ}
If for all $i\in \mathcal{N}$, the reward $r^i$ is upper bounded by $\bar{r}$, then, under Assumption~\ref{aspt:independence}, it holds that for all $s \in \mathcal{S}$ and $a\in \mathcal{A}$,
\begin{equation}
    \left|\mathcal{Q}^i(s^{N^i_\kappa}, a^{N^i_\kappa})- Q^i(s, a) \right| \le \frac{\bar{r}}{1- \gamma} \gamma^{\frac{\kappa}{2}}.
\end{equation}
\end{theorem}
\begin{proof}
Since $\forall i \in \mathcal{N}$, $r^i \le \bar{r}$, the $(\frac{\bar{r}}{1- \gamma}, \sqrt{\gamma})$-exponential decay property holds using Theorem~\ref{lemma:exponential_decay}. According to Definition~\ref{def:exponential_decay}, $\forall i \in \mathcal{N}, \forall s^{N^i_\kappa} \in \mathcal{S}^{N^i_\kappa}, \forall (s^{N^{-i}_\kappa}, s^{N^{-i}_\kappa}_*) \in \mathcal{S}^{N^{-i}_\kappa}, \forall a^{N^i_\kappa} \in \mathcal{A}^{N^i_\kappa}, \forall (a^{N^{-i}_\kappa}, a^{N^{-i}_\kappa}_*) \in \mathcal{A}^{N^{-i}_\kappa}$, we have $|Q^i(s^{N^i_\kappa}, s^{N^{-i}_\kappa}_*, a^{N^i_\kappa}, a^{N^{-i}_\kappa}_*) - Q^i(s^{N^i_\kappa}, s^{N^{-i}_\kappa}, a^{N^i_\kappa}, a^{N^{-i}_\kappa})| \le \frac{\bar{r}}{1- \gamma} \gamma^{\kappa/2}.$ 
That is to say $|Q^i(s^{N^i_\kappa}, s^{N^{-i}_\kappa}_*, a^{N^i_\kappa}, a^{N^{-i}_\kappa}_*) - Q^i(s,a)| \le \frac{\bar{r}}{1- \gamma} \gamma^{\kappa/2}.$ By plugging in $s^{N^{-i}_\kappa}_*$ and $a^{N^{-i}_\kappa}_*$ as $\mathbf{0}$, we have $\left|\mathcal{Q}^i(s^{N^i_\kappa}, a^{N^i_\kappa})- Q^i(s, a) \right| \le \frac{\bar{r}}{1- \gamma} \gamma^{\kappa/2}.$
\end{proof}

Theorem~\ref{lemma:truncatedQ} shows that the approximation error of the truncated action value function $\mathcal{Q}^i(s^{N^i_\kappa}, a^{N^i_\kappa})$ is bounded by $\frac{\bar{r}}{1- \gamma} \gamma^{\kappa /2}$. As the communication technology develops, the $\kappa$ becomes larger and this approximation error decreases exponentially small. Information shared from neighbors shows benefit in getting a scalable and well approximated $\mathcal{Q}$-function for MARL-based CAV behavior learning. Results in Theorem~\ref{lemma:exponential_decay} and Theorem~\ref{lemma:truncatedQ} show that the ego vehicle only needs to get the states and actions of its neighboring vehicles rather than that of all CAVs to learn the truncated $\mathcal{Q}$-function.


\begin{figure*}[h]
  \centering
  \includegraphics[width=1.5\columnwidth]{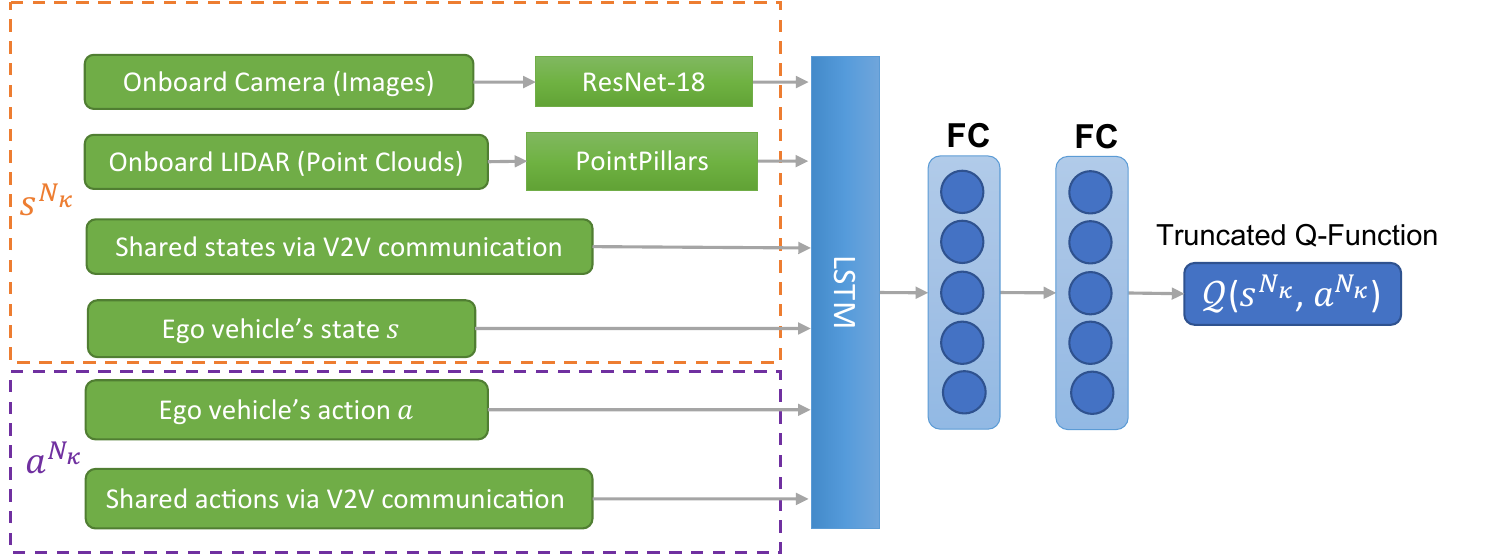}
  \vspace{-8pt}
  \caption{The truncated $\mathcal{Q}^i(s^{N^i_\kappa}, a^{N^i_\kappa})$ network for the behavior planning with the LSTM (long short-term memory) layer and FC (fully connected) layers. We use truncated $\mathcal{Q}$-function to approximate the centralized critic such that the training process utilizes the information sharing capability of CAVs instead of relying on the global states and actions.}\label{fig:Qnet}
   \vspace{-10pt}
\end{figure*}

\subsection{Truncated $\mathcal{Q}$-network}

The $\mathcal{Q}^i(s^{N^i_\kappa}, a^{N^i_\kappa})$-network we use for the behavior planning is shown in Fig.~\ref{fig:Qnet}. We consider the action set $\mathcal{A}$ includes \{Keep Lane (KL), Change Left (CL), and Change Right (CR)\}. The inputs of the truncated $\mathcal{Q}$-network include: 
\begin{itemize}
    \item Ego vehicle's vision information including onboard camera images and LIDAR data (point clouds);
    \item Shared states (including shared vision information) and actions from neighboring CAVs via V2V;
    \item Ego vehicle's state $s^i$ and action $a^i$.
\end{itemize}

The input design is based on Theorem~\ref{lemma:truncatedQ}. Since we can use truncated $\mathcal{Q}^i(s^{N^i_\kappa}, a^{N^i_\kappa})$ to approximate $Q^i(s, a)$ with bounded approximation error, the input of the $\mathcal{Q}^i(s^{N^i_\kappa}, a^{N^i_\kappa})$-network does not rely on the global $s$ and the global action $a$. Meanwhile, the joint state and action spaces do not grow even in a large-scale CAV system. In this $\mathcal{Q}$-network, we adopt the  ResNet-18~\cite{he2016deep} structure to process the images and the PointPillars~\cite{lang2019pointpillars} to process the point clouds. The implementation details will be introduced in the experiment section.


\subsection{Weight-pruned CNN For Truncated $\mathcal{Q}$-function Learning}
\label{sec:cnn_vision}
When updating the truncated $\mathcal{Q}^i(s^{N^i_\kappa}, a^{N^i_\kappa})$-network, we need the $\kappa$-hop neighborhood's states including vision information captured by onboard cameras and LIDAR sensors. Therefore, we assume all CAVs share their states by V2V communication. However, it is unrealistic for CAVs to share the raw camera images and point clouds with others due to the following two reasons: (i) the limited bandwidth of a V2V link prevents sharing the raw camera and LIDAR information among vehicles~\cite{miller2020cooperative}; (ii) the same shared information on neighboring vehicles needs to be repeatedly learned for lane information extraction, resulting in computing resource waste.

To overcome these challenges, we first process the vision information locally using CNNs and then share the extracted features with neighboring vehicles. We observe that the processing time of point cloud data is significantly longer (11.16$\times$) than that of images. The point cloud data processing becomes the ``critical path" of the overall vision process since the raw images and point clouds need to be synchronously processed and shared for behavior planning. Previous research has shown that there exists redundancy in CNN model parameters~\cite{iandola2016squeezenet,luo2017thinet}. 
Hence, we further develop a weight pruning technique to speed up the slower process. 

CNN {weight pruning} can be used to exploit the redundancy in the parameterization of deep architectures while maintaining the CNN model accuracy.  We formulate the weight pruning problem in an $N$-layer CNN as:
\begin{equation}
\small
\label{equ1}
\begin{aligned}
& \underset{ \{{\bf{W}}_{i}\}}{\text{minimize}}
& & \mathcal{F} \big( \{{\bf{W}}_{i}\}_{i=1}^N \big),
\\ & \text{subject to}
& & \mathrm{cardinality}({\bf{W}}_{i})\le l_{i}, \; i = 1, \ldots, N,
\end{aligned}
\end{equation}
where ${\bf{W}}_{i}$ represents the weights in the $i$-th layer. {$\mathcal{F}$ is the CNN loss function with respect to ${\bf{W}}_{i}$.
Here, we use the cross entropy loss as the CNN loss to measure the difference between the real label and the predicted label. We seek to minimize the loss function so that we can obtain a candidate solution that has the lowest error~\cite{ian2016deep}.} We use $l_{i}$ to represent the desired numbers of non-zero weights~\cite{zhang2018systematic}. The key idea is to represent a neural network with a simpler model through pruning the redundant weights (whose magnitudes are below a threshold). Then we retrain the CNN to fine-tune the weights of the remaining connections. 
Our vision processing method is used to extract features of both autonomous and human-driven vehicles in the mixed traffic environment. The implementation details of the weight-pruned CNN will be introduced in the experiment section. 


\section{Behavior Planning and Safety}
\label{sec:learning_control}

In this section, we introduce our main contribution, the safe actor-critic MARL algorithm, to learn a behavior planning policy $\Phi(\pi)$ to select actions. A safe action mapping $\Phi$ based on the control barrier function (CBF) is designed to solve the last challenge mentioned in the introduction to guarantee that the action explored in training and execution is safe with feasible control inputs. The novel safe actor-critic algorithm and the resulting localized policy design can utilize the unique strength of CAVs to gather more information for decision making based on V2V communication.  
We also introduce the control barrier function (CBF)-based quadratic programming (QP) controller in subsection $B$ to generate control inputs for steering angle and acceleration with provable safety.


\begin{algorithm}[h]
\SetAlgoLined
  Input: action set $\mathcal{A}$, action $a^i \in \mathcal{A}$ \;
  Output: safe action $a^i \in \mathcal{C}_{\mathcal{A}} \subseteq \mathcal{A}$ \;
  
  \While{True}
  {
    \eIf{$a^i$ is safe, i.e., the CBF-QP (introduced in subsection $B$) has a feasible solution when plugging in $a^i$} 
    {
        \Return $a^i$\;
    }
    {
        Remove $a^i$ from $\mathcal{A}$\;
        \eIf{$\mathcal{A}$ is not empty}
        {
            find action in $\mathcal{A}$ with the highest $\mathcal{Q}$ value and assign it to $a^i$\;
        }
        {
            \Return $a^i = ES$, emergency stop~\cite{khelladi2020emergency}\;
        }
    }
  }
 \caption{Safe Action Mapping $\Phi$}
 \label{alg:safe_action}
\end{algorithm}
\setlength{\textfloatsep}{1pt}

\subsection{Behavior Planning Algorithm}
\label{sec:safe_algorithm}

Since autonomous driving is a safety-critical application, the behavior policy used to explore different actions must be designed rigorously to avoid potential accidents. The action $a$ generated by the traditional policy for exploration, like the $\epsilon$-greedy method~\cite{sutton2018reinforcement}, may not be safe to execute without additional checking. For example, lane-changing may lead to collisions with neighboring vehicles. To ensure that there are no collisions during the training process, we design a safe action mapping function $\Phi$ to map any action from the action set $\mathcal{A}$ to the safe action set $\mathcal{C}_{\mathcal{A}} \subseteq \mathcal{A}$. 

\subsubsection{Safe Action Mapping}
\label{sec:safe_action}

The safe action mapping $\Phi$ is shown in Alg.~\ref{alg:safe_action}. We use a CBF-QP controller (it is to be introduced in subsection $B$) to evaluate whether an action is safe or not. If $a$ is safe, then return the safe action $a$; if not, the controller will search other actions in $\mathcal{A}$ in descending order according to their action value and find a safe one. If all the actions in $\mathcal{A}$ are not safe in the worst case, then the controller will apply the emergency stop (ES) process. 
In case when the safe action set $\mathcal{C}_\mathcal{A}$ is empty, an emergency stop will be executed, and the corresponding state will be marked unsafe and receive a large penalty. The emergency stop will only be performed in an emergency scenario where all normal actions are not safe~\cite{khelladi2020emergency}.

\subsubsection{Safe Actor-Critic Algorithm}


\begin{algorithm}[h]
\SetAlgoLined
 Randomly initialize the critic network $\mathcal{Q}^i$ and the actor network $\pi^i$ for agent $i$. Initialize target networks $\mathcal{Q}'^{i}, \pi'^{i}$. Initialize replay buffers $\mathcal{D}^i$\;
 \For {each episode}
 {
    Initialize the global state $s$\;
    \For {each timestep}
    {
        With probability $\epsilon$, select action $a^i= \Phi (\pi^i(s^{N^i}))$ for each agent $i$, where $\Phi$ is the safe action mapping in Alg.~\ref{alg:safe_action}. With probability $1 - \epsilon$, select an action $a^i$ for each agent $i$ in $\{ \Phi(a^i)|a^i \in \mathcal{A} \}$ randomly\;
        
        
        Execute actions $a = (a^1,...,a^n)$ and observe reward $r$ and new state $s'$ and set $s \leftarrow s'$\;
        
        \For {each agent $i$}
        {
            Store $(s^{N^i_\kappa}, a^{N^i_\kappa}, r^i, s'^{N^i_\kappa})$ in replay buffer $\mathcal{D}^i$\;
            
            Sample a random minibatch from $\mathcal{D}^i$\;
            
            Set $y_k^i = r^i_k + \gamma \mathcal{Q}'^{i}(s_k'^{N^i_\kappa}, a_k'^{N^i_\kappa})|_{a'^{i} = \pi'^{i}(s'^{N^i})}$\;
            
            Update critic by minimizing the loss $\mathcal{L}(\theta^i) = \frac{1}{K}\sum_k (y_{k}^i - \mathcal{Q}^{i}(s_k^{N^i_\kappa}, a_k^{N^i_\kappa}))^2$\;
            
            Update actor using the gradient $\nabla_{\theta^i}J = \frac{1}{K}\sum_k \nabla_{\theta^i} \pi^i(s^{N^i})\nabla_{a^i} \mathcal{Q}^i(s_k^{N^i_\kappa}, a_k^{N^i_\kappa})$ where $a^i = \pi^i(s^{N^i})$\;
            
            Update target networks: $\theta'^{i} \leftarrow \tau\theta^i + (1-\tau)\theta'^{i}$.
        }
    }
 }
 \caption{Safe Actor-Critic Algorithm}
 \label{alg:actor-critic}
\end{algorithm}
\setlength{\textfloatsep}{1pt}

In our behavior planning algorithm design, we use a decentralized training and decentralized execution paradigm. Each CAV learns a truncated critic $\mathcal{Q}^i(s^{N_\kappa^i}, a^{N_\kappa^i})$ and then use it to train a localized policy $\Phi (\pi^i(a^i|s^{N^i}))$ for decentralized execution. We adopt an actor-critic structure so that the actor and the critic can use different information sets. We do not derive a policy directly from the truncated $\mathcal{Q}^i(s^{N^i_\kappa}, a^{N^i_\kappa})$ function but design an actor besides the critic, because we want the policy for each CAV to only depend on its available states information excluding the $\kappa$-hop neighbors' actions. Otherwise, there will be a deadlock during execution, and CAVs will wait for each other's actions.

Our main design, the safe actor-critic algorithm, is shown in Alg.~\ref{alg:actor-critic}. We mainly design two new techniques in our algorithm:
\begin{itemize}
    \item \textit{Truncated $\mathcal{Q}$-function}: Each vehicle learns a truncated $\mathcal{Q}$-function in decentralized training such that the training process does not rely on the global states and the global actions. The joint state and action spaces of the truncated $\mathcal{Q}$-function do not grow in a large-scale CAV system.
    \item \textit{Safe action mapping}: We map any action in the action space to the safe action set so that the training and execution have provable safety guarantees.
\end{itemize}

The transition experience in Alg.~\ref{alg:actor-critic} is represented by $(s^{N_\kappa^i}, a^{N_{\kappa}^i}, r^i, s'^{N_\kappa^i})$, where $s^{N_\kappa^i}$ and $a^{N_{\kappa}^i}$ are the current states and actions of the $\kappa$-hop neighborhood including vehicle $i$, $r^i$ is the reward, $s'^{N_\kappa^i}$ is the next states of the $\kappa$-hop neighborhood including vehicle $i$. This algorithm learns a truncated critic $\mathcal{Q}^i(s^{N_\kappa^i}, a^{N_\kappa^i})$ by minimizing the Bellman loss: 
\begin{equation}
    \mathcal{L}(\theta^i) = \Expect_{s^{N_\kappa^i}, a^{N_{\kappa}^i}, r^i, s'^{N_\kappa^i}} \left[(y - \mathcal{Q}^i(s^{N_\kappa^i}, a^{N_{\kappa}^i}; \theta^i))^2 \right],
\end{equation}
where $y^i = r^i + \gamma \cdot \mathcal{Q}'^i(s'^{N_\kappa^i}, a'^{N_{\kappa}^i}; \theta'^i)|_{a'^{i} = \pi'^{i}(s'^{N^i})}$, $\gamma$ is the discount factor, $\mathcal{Q}'^i$ is the target network for the critic, $\pi'^i$ is the target network for the actor. The safe action mapping assures the policy used to produce a new transition experience is safe. The generated experience is stored in a replay buffer $\mathcal{D}^i$. When training the $\mathcal{Q}^i$-network, a mini-batch is sampled from $\mathcal{D}^i$ to decorrelate data. 
Then we use this truncated critic to train a localized actor using the following gradient
\begin{equation}
    \nabla_{\theta^i}J =  \Expect_{s^{N_\kappa^i}, a^{N_\kappa^i} \sim \mathcal{D}} \left[ \nabla_{\theta^i} \pi^i(s^{N^i})\nabla_{a^i}\mathcal{Q}^i(s_k^{N_\kappa^i}, a_k^{N_\kappa^i}) \right].
\end{equation}


\begin{remark}
The safe actor-critic algorithm in Alg.~\ref{alg:actor-critic} is scalable for a large-scale CAV system for the following two reasons: (i) The Alg.~\ref{alg:actor-critic} learns a truncated action value function $\mathcal{Q}^i(s^{N_\kappa^i}, a^{N_{\kappa}^i})$ that only requires the states and actions of vehicle $i$'s $\kappa$-hop neighborhood, i.e., the joint state and action spaces of the truncated $\mathcal{Q}$-function do not grow with the total number of CAVs; (ii) The policy $ \Phi (\pi^i (a^i | s^{N^i}))$ only depends on the local states.
\end{remark}

\subsection{Safety For the Ego Vehicle}
\label{sec:controller_safety}

This section introduces the control barrier function (CBF)-based quadratic programming (QP) controller design for the ego vehicle to find the safe control inputs for the steering angle and the acceleration. It is also used in the safe action mapping algorithm (Alg.~\ref{alg:safe_action}) to check the safety of CAVs' actions. Since the controller design is the same for each vehicle $i$, we drop the superscript $i$ when there is no confusion in this section. Before we give the formulation of the CBF-QP controller, we first introduce the concept of control barrier function, safe set, and forward invariant.
\begin{definition}[Control barrier function (CBF) \cite{agrawal2017discrete,safeRL_aaai19}]
A mapping $h: \mathcal{X} \rightarrow \mathbb{R}$ is a discrete-time exponential control barrier function for dynamic system \eqref{equ:low_sys_equation_with_noise} if
\begin{enumerate}
    \item $h(x_0) \geq 0$ and,
    \item there exists a control input $u_t \in \mathcal{U}$ and $\eta \in(0,1]$ such that for all $t \in \mathbb{N}_+$, 
    \begin{equation}
    h(f(x_t) + g(x_t) u_t + w_t) + (\eta - 1) h(x_t) \geq 0,
    \label{equ:CBF}
\end{equation}
\end{enumerate}
\end{definition}
\begin{definition}[Safe set \cite{agrawal2017discrete, ames2019control}]
\label{def:safe_set}
Denote the safe set $\mathcal{C}$ as
\begin{equation}
    \mathcal{C}: \{x \in \mathcal{X} \in \mathbb{R}^{n_x} \mid h(x) \geq 0\},
    \label{equ:safe_set}
\end{equation}
where $h: \mathcal{X} \rightarrow \mathbb{R}$ is a discrete-time exponential control barrier function.
\end{definition}
\begin{definition}[Forward invariant \cite{agrawal2017discrete, ames2019control}]
\label{def:forward_invariant}
For any initial state $x_0 \in \mathcal{C}$, if the state $x_t$ is in $\mathcal{C}$ for all $t \geq 0$, then the set $\mathcal{C}$ is said to be forward invariant. 
\end{definition}

The system is safe with respect to set $\mathcal{C}$ if $\mathcal{C}$ is forward invariant, i.e., the system state $x$ always remains within the safe set. The CBF is a condition to add the forward invariant property to the system's state $x$. For the CAV problem, we consider an affine barrier function with the form $h = p^\mathsf{T} x + q$, $(p \in \mathbb{R}^{n_x}, q \in \mathbb{R})$. This restriction means the set $\mathcal{C}$ is a polytope constructed by the intersecting of half-spaces since the safe area for a vehicle is usually represented by a bounding box \cite{cesari2017scenario}. By applying a Frenét frame, the bounding box also works for curve roads~\cite{li2022autonomous}. 

We formulate the following Quadratic Programming (QP) based on the CBF condition \eqref{equ:CBF} that can be solved efficiently at each timestep:
\begin{equation}
\begin{aligned}
& \underset{u_t, \zeta}{\argmin} & & \norm{u_t - U \cdot a}^2 + M \zeta\\
& \text{s.t.} & & p^\mathsf{T} f(x_t) + p^\mathsf{T} g(x_t) u_t + q - p^{\mathsf{T}} W \cdot \mathbf{1} \\
& & & \geq (1 - \eta) h(x_t) - \zeta \\
& & & \zeta \geq 0,
\label{equ:QP_noise}
\end{aligned}
\end{equation}
where $M$ is a large constant, $a$ is the action to change/keep lane, and $U \cdot a$ is the control reference for action $a$. For any action selected by the policy, we use state-of-the-art trajectory tracking~\cite{dixit2018trajectory,cesari2017scenario} to calculate the control references. The implementation details of the CBF-QP controller will be introduced in the experiment section for our CAV problem.



We have the following theorem to show the control inputs calculated by solving \eqref{equ:QP_noise} can guarantee the safety of the system  $x_{t+1} = f(x_t) + g(x_t) u_t + w_t$ in \eqref{equ:low_sys_equation_with_noise} with bounded noise.

\begin{theorem}
\label{theorem:safety}
When the physical dynamics of each autonomous vehicle satisfy~\eqref{equ:low_sys_equation_with_noise} with bounded noise $\norm{w_t} \leq W$, if there exists $\eta \in (0,1]$ such that the QP~\eqref{equ:QP_noise} has a solution for all $x_t \in \mathcal{C}$ and the solution satisfies $\zeta \leq Z$, then the controller derived from \eqref{equ:QP_noise} renders set $\mathcal{C}': \{x \in \mathbb{R}^{n_x} \mid h'(x) = h(x) + \frac{Z}{\eta} \geq 0\}$ forward invariant.
\end{theorem}
\begin{proof}
Since $x_t \in \mathcal{C}$, we have $h(x_t) \geq 0$ according to the Definition~\ref{def:safe_set}.
From $Z \geq \zeta \geq 0 \geq - \eta h(x_t)$, we have $h(x_t) + \frac{Z}{\eta} \geq 0$. Thus, $x_t \in \mathcal{C}'$ with $h'(x_t) \geq 0$. Also, we have
\begin{equation}
\label{equ:safety_first}
\begin{aligned}
h(x_{t+1}) &= p^\mathsf{T} x_{t+1} + q \\
&= p^\mathsf{T} f(x_t) + p^\mathsf{T} g(x_t) u_t + p^\mathsf{T} w_t + q \\
& \geq p^\mathsf{T} f(x_t) + p^\mathsf{T} g(x_t) u_t + q - p^{\mathsf{T}} W \cdot \mathbf{1} \\
& \geq (1 - \eta) h(x_t) - \zeta \geq (1 - \eta) h(x_t) - Z, \\
\end{aligned}
\end{equation}
where the first inequality follows $\norm{w_t} \leq W$ and the second inequality follows the inequality constraints in the CBF-QP~\eqref{equ:QP_noise}. Adding $Z/\eta$ to both sides of~\eqref{equ:safety_first}, we have
\begin{equation}
\begin{aligned}
h(x_{t+1}) + \frac{Z}{\eta} & \geq (1 - \eta) \left[h(x_t) + \frac{Z}{\eta} \right], \\
h'(x_{t+1}) & \geq (1 - \eta) h'(x_t) \geq 0.
\end{aligned}
\end{equation}
Thus, $x_{t+1} \in \mathcal{C}'$ and $\mathcal{C}'$ is forward invariant.
\end{proof}

The value of $Z$ denotes how large the CBF condition \eqref{equ:CBF} is violated from the original $h(x_t)$. In this case, the safety condition should be formulated according to the set $\mathcal{C}'$. 

Theorem~\ref{theorem:safety} shows that the safe set $\mathcal{C}'$ is forward invariant using the control inputs calculated by the CBF-QP in~\eqref{equ:QP_noise}, that is to say, the ego vehicle's control states are always in the safe set and the ego vehicle is guaranteed to be safe. We also use \eqref{equ:QP_noise} in Alg.~\ref{alg:safe_action} for a safety checking by plugging in the corresponding action $a$. As long as there is a feasible solution to \eqref{equ:QP_noise}, we can find control inputs for this behavior action $a$ such that the control states are always in the safe set, in other words, the action $a$ is a safe action. 
In Appendix~\ref{sec:CBF_QP}, we show how to formulate the CBF-QP for a system when we do not need to consider the noise. 

\subsection{Safety For the Multi-Agent System}
At last, we show the connection between the safety of a single agent and the multi-agent system. Note that the CBF-QP in~\eqref{equ:QP_noise} is defined for a single CAV, and $x_t$ in~\eqref{equ:QP_noise} only includes the control state for a single CAV. Theorem~\ref{theorem:safety} shows the safety guarantee of each CAV. Consider the multi-agent system with $n$ agents, the joint control state is $\mathbf{x} = (x^1, ..., x^n) \in \mathbf{X} \defeq \mathcal{X}^1 \times \cdots \times \mathcal{X}^n$, the joint control input is $\mathbf{u} = (u^1, ..., u^n) \in \mathbf{U} \defeq \mathcal{U}^1 \times \cdots \times \mathcal{U}^n$. The safe set defined in Definition~\ref{def:safe_set} can be extended for the joint control state as:
\begin{equation}
    \mathbf{C}: \{\mathbf{x} \in \mathbf{X} \in \mathbb{R}^{n \times n_x} \mid \forall i \in \mathcal{N}, h^i(x^i) \geq 0\}.
    \label{equ:jointsafe_set}
\end{equation}
The safety of a multi-agent system can be formally defined as follows:
\begin{definition}[Safety of a Multi-Agent System]
\label{def:safety_multi}
For any initial state in the safe set $\mathbf{x}_0 \in \mathbf{C}$, if the joint control state $\mathbf{x}_t$ is in $\mathbf{C}$ for all $t \geq 0$, then the multi-agent system is safe with respect to the safe set $\mathbf{C}$.
\end{definition}
We use a decentralized manner to find the safe control inputs for each agent: at each time $t$, each agent solves its own CBF-QP in~\eqref{equ:QP_noise} to find the safe control inputs for itself. Then we have the following Proposition~\ref{prop:safety_multi} to show the safety of the multi-agent system when all agents follow the CBF-QP in~\eqref{equ:QP_noise}.
\begin{prop}
\label{prop:safety_multi}
If there exists $\eta^i \in (0,1]$ such that the QP~\eqref{equ:QP_noise} for each agent $i \in \mathcal{N}$ has a solution for all $x^i_t \in \mathcal{C}^i$ and the solution satisfies $\zeta^i \leq Z^i$, then the controller derived from \eqref{equ:QP_noise} for each agent guarantees the safety of the multi-agent system with respect to $\mathbf{C}': \{\mathbf{x} \in \mathbb{R}^{n \times n_x} \mid \forall i \in \mathcal{N}, h^{i\prime}(x^i) = h^i(x^i) + \frac{Z^i}{\eta^i} \geq 0\}$.
\end{prop}
\begin{proof}
According to Theorem~\ref{theorem:safety}, for each $i \in \mathcal{N}$, the controller derived from \eqref{equ:QP_noise} for each agent renders the safe set $\mathcal{C}^{i\prime}: \{x^i \in \mathbb{R}^{n_x} \mid h^{i\prime}(x^i) = h^i(x^i) + \frac{Z^i}{\eta^i} \geq 0\}$ forward invariant. Therefore, starting from any safe joint state $\mathbf{x}_0 \in \mathbf{C}'$, the control state $x^i_t$ is within $\mathcal{C}^{i\prime}$ for all $i \in \mathcal{N}$ and all $t \geq 0$ according to the forward invariant property defined in Definition~\ref{def:forward_invariant}. Thus, the joint control state $\mathbf{x}_t$ is within the $\mathbf{C}': \{\mathbf{x} \in \mathbb{R}^{n \times n_x} \mid \forall i \in \mathcal{N}, h^{i\prime}(x^i) = h^i(x^i) + \frac{Z^i}{\eta^i} \geq 0\}$, and the multi-agent system is safe with respect to $\mathbf{C}'$.
\end{proof}
Proposition~\ref{prop:safety_multi} shows that each agent only needs to care about its own safety if all agents follow the same form of safety rules, i.e., using the CBF-QP in~\eqref{equ:QP_noise} to find control inputs, and the entire multi-agent system will be safe. This decentralized controller design can significantly improve the scalability of the CAV system and avoid using a centralized controller with a combinatorially large joint control state space.

\section{Experiment}
\label{sec:experiment} 

In this section, we first show the implementation details and the evaluation of the weight-pruned CNN in subsection A. In subsection B, we show our safe actor-critic algorithm improves traffic efficiency under different CAV ratios and different traffic densities. We also evaluate the truncated $\mathcal{Q}$-function technique in this set of experiments. Then in subsection C, we show our approach guarantees the safety of the CAVs. In subsection D, we construct an obstacle-at-corner scenario to show the benefit of the shared vision.
In all experiments, we use CARLA~\cite{Dosovitskiy17}, an open-source simulator that supports the development training, and validation of autonomous driving systems, to validate our proposed safe actor-critic algorithm. 

\begin{figure}[h]
  \centering
  \includegraphics[width=1.8 in]{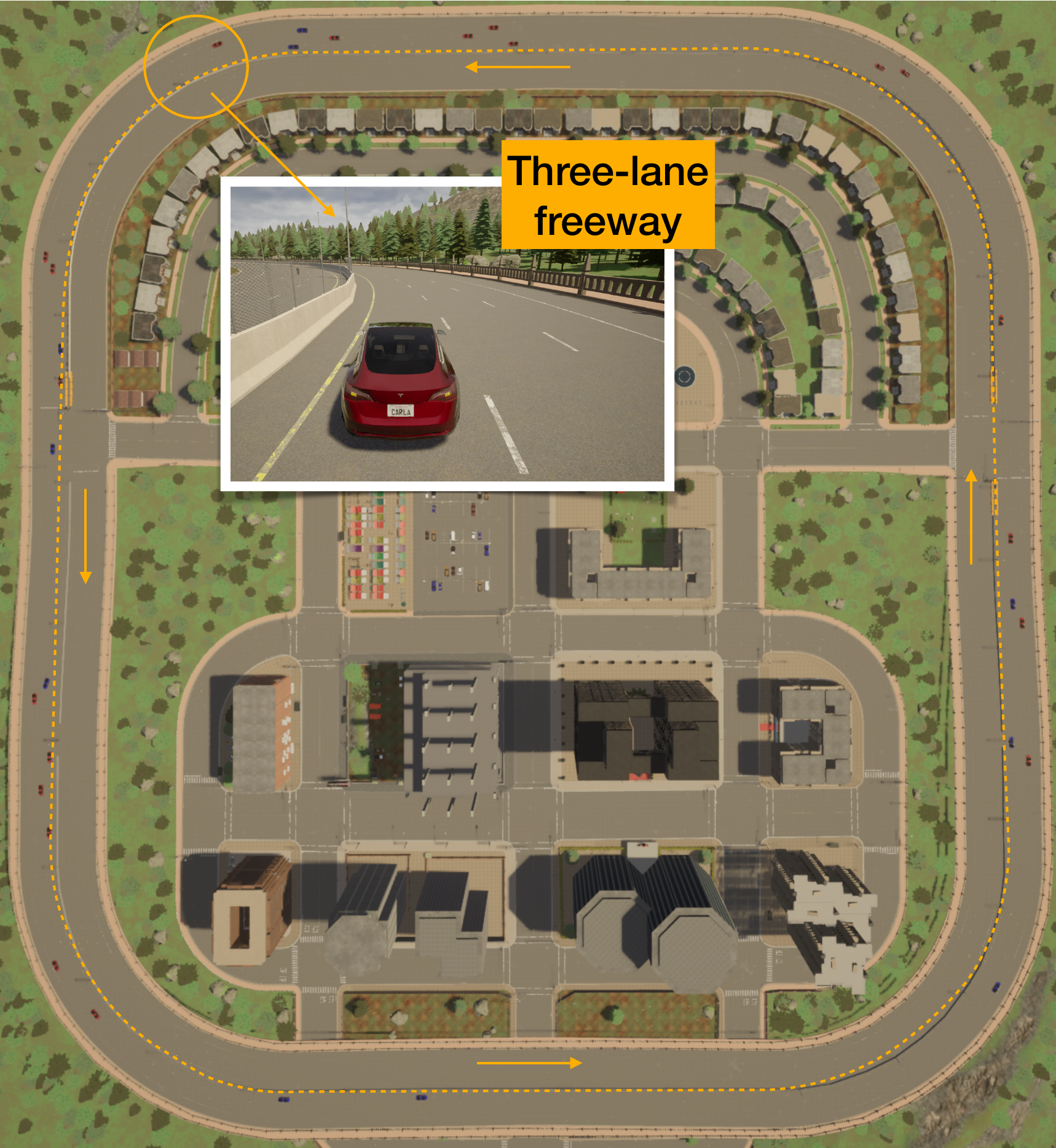}
  \vspace{-10pt}
  \caption{The example scenario of a 3-lane freeway in CARLA~\cite{Dosovitskiy17}. Vehicles are scattered on the outer loop of the map ``Town 05". The environment is mixed with both autonomous and human-driven vehicles.}\label{fig:freeway_carla}
\end{figure}

All vehicles run on a 3-lane freeway as shown in Fig.~\ref{fig:freeway_carla}. Each CAV in CARLA is integrated with a camera sensor for capturing RGB images, and a LIDAR sensor for generating point clouds. The resolution of camera image is $375\times1,242$ pixels. We set row anchors ranging from 100 to 370 with intervals of 10, and the number of grids is 100. We scale each image to $288\times800$. Each point cloud from LIDAR is stored with the 3 coordinates (the ego vehicle being the origin), representing forward, left, and up respectively, and an additional reflectance value. To augment the dataset, we apply random mirroring flip along the forward axis, and a global rotation and scaling. In evaluation, we set forward and up axes' resolution to $0.16 m$, max number of pillars to 12,000, and max number of points per pillar to 100. Images and point clouds are paired under the same timestamp and are processed jointly under the CNN. 

We assume all CAVs share their states (including vision information) and actions with others. Each CAV uses minibatch gradient descent to learn the truncated $\mathcal{Q}^i(s^{N^i_\kappa}, a^{N^i_{\kappa}})$ critic introduced in Section~\ref{sec:info_sharing}. Then each CAV learns a localized policy $\Phi(\pi^i(a^i|s^{N^i}))$ based on the learned critic with a safe action mapping $\Phi$ introduced in Section~\ref{sec:learning_control}. The detail of this safe actor-critic algorithm is introduced in Section~\ref{sec:safe_algorithm}. The hyperparameters of our safe actor-critic algorithm are shown in Table~\ref{tbl:hyperparameter}.

\begin{table}[h!]
\centering
\caption{Safe Actor-Critic Hyperparameters}
\begin{tabular}{l|c}
\toprule
\textbf{Parameter}                         & \textbf{Value}\\ 
\midrule
optimizer                         & Adam                  \\ 
learning rate                     & 0.01                  \\ 
discount factor                   & 0.9                  \\ 
replay buffer size                & $10^6$                  \\ 
hidden size in FC and LSTM layers    & 128                    \\ 
minibatch size  & 64                 \\ 
activation function                      & ReLU                  \\ 
\bottomrule
\end{tabular}
\label{tbl:hyperparameter}
\end{table}

The host machine adopted in our experiments is a server configured with Intel Core i9-10900X processors and four NVIDIA RTX2080Ti GPUs. Our experiments are performed on Python 3.7.6, GCC7 7.5, PyTorch 1.6.0, and CUDA 11.0.

\subsection{Safe Actor-Critic Algorithm Implementation Details}
\subsubsection{Reward Function}
In the implementation of our safe actor-critic algorithm, we consider a stage-wise reward for the ego vehicle to improve velocity $v^i$ and comfort $c^i$. The comfort of a vehicle (for passenger's experience) is defined based on its acceleration and action as follows:



\begin{equation}
\label{equ:comfort}
\small
    c^i(\dot{v}^i, a^i) = \begin{cases} 
			 3, & \text{if } \lvert \dot{v}^i \rvert< \Theta \text{ and } a^i = KL; \\
			2, & \text{if } \lvert \dot{v}^i \rvert \geq \Theta \text{ and } a^i = KL; \\
			 1, & \text{if } a^i = CL/CR; \\
			 0, & \text{if in } ES. \end{cases} 
\end{equation}
where $\dot{v}^i$ is the acceleration, $a^i \in \mathcal{A}$ is the MARL action, $\Theta$ is a predefined threshold.
The reward function for vehicle $i$ is defined as:
\begin{equation}
r^i(s^i, a^i) = \omega \cdot v^i + c^i(\dot{v}^i, a^i),
\label{equ:reward}
\end{equation}
where $\omega$ is a positive trade-off weight.

\subsubsection{The CBF-QP}
In the implementation of the CBF-QP controller~\eqref{equ:QP_noise}, the details of the system model are in Appendix~\ref{sec:physical_model}. The control state $x_t = [p_x, p_y, \psi, v]^\mathsf{T} \in \mathcal{X}$ includes the two coordinates of the center of gravity (CoG), its orientation and velocity. The control input $u_t = [\tan(\delta), \dot{v}] \in \mathcal{U}$ includes the tangent of the steering angle and the acceleration.
The action $a = \{a_j | a_j \in \mathcal{A}\}$ is selected based on the policy learned by the safe actor-critic algorithm. In the controller module, we use standard basis vectors in $\mathbb{R}^{|\mathcal{A}|}$ to denote each discrete action, where $|\mathcal{A}|$ is the number of total actions in the action set $\mathcal{A}$. In our problem, the three actions are represented as $(1,0,0)^\mathsf{T}, (0,1,0)^\mathsf{T}, (0,0,1)^\mathsf{T}$.


For each action selected by the behavior planning module, we can use the state-of-the-art trajectory planning methods to generate trajectories $\sigma = [\sigma_1, \sigma_2]^\mathsf{T} = [p_x, p_y]^\mathsf{T}$, such as potential fields, cell decomposition, model predictive control (MPC)~\cite{dixit2018trajectory,cesari2017scenario}. We use endogenous transformation to compute the controller's references to track trajectories:
\begin{align}
    u_{ref1} = \frac{\dot{\sigma}_1 \ddot{\sigma}_1 + \dot{\sigma}_2 \ddot{\sigma}_2}{ \sqrt{\dot{\sigma}_1^2 + \dot{\sigma}_2^2}}, \quad
    u_{ref2} = \frac{\dot{\sigma}_1 \ddot{\sigma}_2 - \dot{\sigma}_2 \ddot{\sigma}_1}{ (\dot{\sigma}_1^2 + \dot{\sigma}_2^2)^{\frac{3}{2}}}.
\end{align}
Assembling them in the matrix $U$ as follows for all actions $a \in \mathcal{A}$:
\begin{equation}
    U = \begin{bmatrix}
    u_{1,ref1} & u_{2, ref1} & \cdots & u_{|\mathcal{A}|,ref1} \\
    u_{1, ref2} & u_{2, ref2} & \cdots & u_{|\mathcal{A}|,ref2} \\
    \end{bmatrix},
    \label{equ:reference}
\end{equation}
where $u_{j,ref1}, u_{j,ref2}$ is the references for $j$-th action $a_j$. We can retrieve the references for $a_j$ by multiplying it with $U$ in the objective of the CBF-QP in Eq.~\eqref{equ:QP_noise}.

\begin{figure}[h]
  \centering
  \includegraphics[width=2.8in]{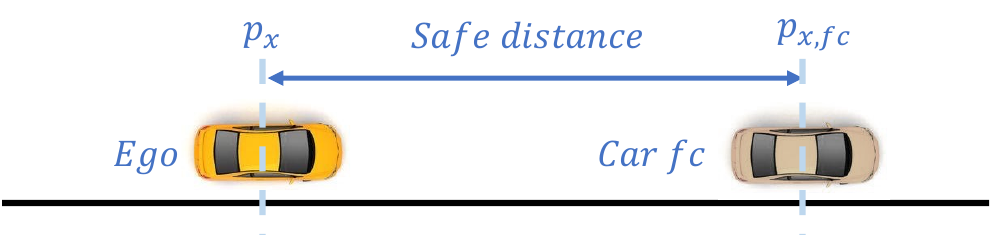}
  \caption{The ego vehicle should keep a safe distance from the front vehicle when keeping lane. To show the idea of safe distance, each vehicle is treated as a mass point at the CoG in this figure.}\label{fig:cbf_keeping}
\end{figure}

Then we show how we define the CBF for each agent. With respect to the ego vehicle, vehicle $fc$ represents the vehicle immediately in front of the ego vehicle in its \textit{current lane} as shown in Fig.~\ref{fig:cbf_keeping}.
When the action of the ego vehicle is changing lane, vehicle $bt$ represents the vehicle immediately behind the ego
vehicle in the \textit{target lane}, vehicle $ft$ denotes the vehicle immediately in front of the ego vehicle in the target lane as shown in Fig.~\ref{fig:cbf_change}. We use $\triangle p_{x,k}= \|p_x-p_{x, k}\|$ to denote the distance between ego vehicle and vehicle $k \in \{fc, ft, bt\}$. Then we define the CBF along the $p_x$-axis as follows:
when the action of the ego vehicle is keeping lane, the CBF is:
\begin{equation}
    h_{fc}(x)  =  \triangle p_{x,{fc}} - P_{safe},
\end{equation}
where $P_{safe}$ is the safe distance between the ego vehicle and other vehicles. When the action of the ego vehicle is changing lane, the CBF is: 
\begin{align}
    h_{fc}(x)  &=  \triangle p_{x,{fc}} - P_{safe}, \nonumber \\
    h_{ft}(x)  &=  \triangle p_{x,{ft}} - P_{safe},\\
    h_{bt}(x)  &=  \triangle p_{x,{bt}} - P_{safe}, \nonumber
\end{align}
For the $p_y$-axis, we use CBF to limit the ego vehicle to move within the range $[P^{min}_y, P^{max}_y]$ determined by the width of the lanes:
\begin{align}
    h_{y, max}(x)  &=  -p_y + P^{max}_{y}, \nonumber \\
    h_{y, min}(x)  &=  p_y - P^{min}_y.
\end{align}

In the implementation of the CBF-QP controller~\eqref{equ:QP_noise}, we use a vectorized CBF. When the action of the ego vehicle is keeping lane, $h(x) \allowbreak = \allowbreak [h_{fc}(x), \allowbreak h_{y, max}(x), \allowbreak h_{y, min}(x)]^\mathsf{T}.$ When the action of the ego vehicle is changing lane, $h(x) \allowbreak = \allowbreak [h_{fc}(x), \allowbreak h_{ft}(x), \allowbreak h_{bt}(x), \allowbreak h_{y, max}(x), \allowbreak h_{y, min}(x)]^\mathsf{T}.$


\begin{figure}[h]
  \centering
  \includegraphics[width=3in]{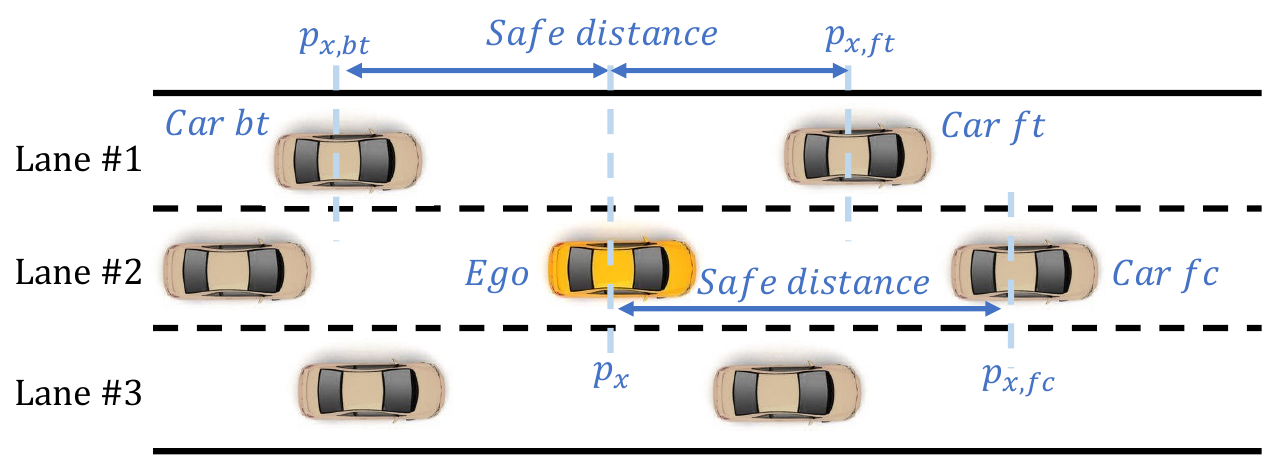}
  \vspace{-10pt}
  \caption{The ego vehicle is changing from lane \#2 to lane \#1. Its safety distance is considered on both the current lane (\#2) and target lane (\#1). In this figure, the length of each car is ignored for simplicity and the value of safe distance already considers the length of the cars.}\label{fig:cbf_change}
  \vspace{-10pt}
\end{figure}

\subsection{CNN-Driven Shared Vision}
\label{sec:vision_experiment}

\begin{figure} 
\centering \includegraphics[width=\columnwidth]{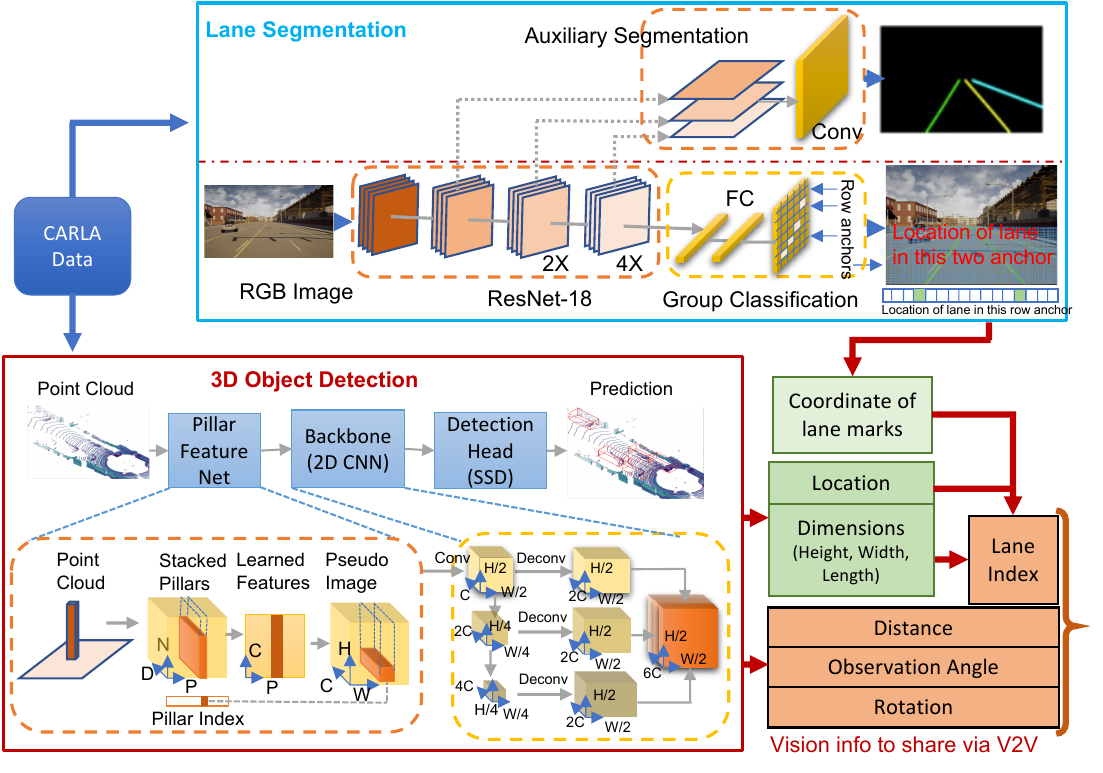}
\vspace{-15pt}
\caption{\label{vision_plot} Detailed structure of vision information processing. We combine the lane segmentation and 3D object detection results to extract neighboring vehicles' features. The processed information is included in the state of each CAV.
}
\end{figure}

We develop a CNN-driven shared vision solution to process and share the vision information to support the learning of the truncated $\mathcal{Q}$-function, as shown in Fig.~\ref{vision_plot}.
{The shared vision model has two main branches. The first one is lane segmentation and the second one is 3D object detection (OD).  We use the network architecture from Ultra-Fast-Lane-Detection~\cite{qin2020ultra} for lane segmentation. 
It aggregates auxiliary segmentation tasks to utilize multi-scale features.} {Table~\ref{net_speed} summarizes the evaluation metrics of several popular CNN backbones for object detection and segmentation that are applied on ImageNet~\cite{net_summary}, including the total number of parameters, test accuracy, and multiply-accumulate (MAC) operation (accumulate the product of two numbers in a CNN and can evaluate the computational complexity of each network~\cite{xie2017aggregated}). {We adopt ResNet-18~\cite{he2016deep} as the backbone since it not only achieves high accuracy but also keeps a small number of parameters and MAC operations.}
The lightweight feature leads the model easier to be deployed for fast inference in processing environment information. 
} 

\begin{table}[ht]
\centering
\caption{Speed and accuracy comparison of several popular CNN backbones. We adopt ResNet-18 because it achieves high accuracy with fewer parameters for fast inference.}
\label{net_speed}
\begin{tabular}{c|ccc}
\toprule
CNN Backbone        & No. of Parameters & MACs & Accuracy \\
\midrule
AlexNet~\cite{krizhevsky2012imagenet}   & 61.1M   & 0.72 & 56.55 \\
\textbf{ResNet-18}~\cite{he2016deep}    & 11.69M  & 1.82 & 71.78 \\
ResNet-50~\cite{he2016deep}    & 25.56M  & 4.12 & 76.15 \\
VGG-16~\cite{simonyan2014very} & 138.36M & 15.5 & 71.59 \\
GoogLeNet~\cite{szegedy2015going} & 42.71M  & 1.52 & 69.78 \\
\bottomrule
\end{tabular}
\end{table}

In the lane segmentation model, the RGB images from the onboard camera are divided into equal-sized grids, where grids in the same rows are defined as row anchors. Grids with the appearance of the lane mark on this row anchor will be colored green. Detailed pixel coordinates of all four-lane marks of the 3-lane freeway will be generated by the neural network and saved in a hash table. 

The 3D OD takes LIDAR point cloud as input and uses the architecture from PointPillars~\cite{lang2019pointpillars} with PointNets~\cite{qi2017pointnet} as backbone. 
The inputs are converted to sparse pseudo-images through an encoder, which can learn a set of features from the stacked pillars and then scatter those features back to 2D pseudo-images.
A network similar to~\cite{zhou2018voxelnet} is used to process the pseudo-image into high-level representation by first continuously down-sampling learned features to small spatial resolution, then up-sampling and concatenating each down-sampled feature. The detection head, single shot multibox detector (SSD)~\cite{liu2016ssd}, is used to predict 3D bounding boxes for neighboring vehicles with the features extracted from the backbone. Finally, the vision information (including the location of vehicles in camera coordinates, dimensions, observation angle, distance, and rotation) will be generated.


{As the camera RGB image and LIDAR point cloud are paired under the same timestamp, we derive the lane index of each neighboring vehicle by fusing the location and dimension of the neighboring vehicle (obtained by 3D OD) with the coordinate of lane marks obtained by lane segmentation.} We then send all vision information including lane index, distance, observation angle, and rotation info to the safe actor-critic algorithm for behavior planning.

{For evaluation, PointPillars uses mean average precision (mAP) as standard, while Ultra-Fast-Lane-Detection uses ``accuracy":} $\frac{\sum_{\text {clip}} C_{\text {clip}}}{\sum_{\text {clip}} S_{\text {clip}}}$, where $C_{clip}$ is the number of lane points predicted correctly and $S_{clip}$ is the total number of ground truth in each clip. We show the performance (accuracy or mAP) and running speed in Table~\ref{vision_result}. 
The lane segmentation achieves high accuracy with low latency. It can process $313$ images per second. The 3D object detection has a running time of $28$ images per second (img/s), which means $100$ images can be processed in around $3.6$ seconds.
Our weight pruning technique significantly increases the speed by $3.5\times$ with a very small accuracy degradation. It decreases the processing time of $100$ images from $3.6$ seconds to one second. Overall, for parallel lane segmentation and 3D object detection, our weight pruning technique significantly speeds up the CNN-driven shared vision process to achieve real-time processing.
Our vision processing method can do lane segmentation and object detection in mixed traffic with both autonomous and human-driven vehicles. 

\begin{table}[ht]
\centering
\caption{Summary of results from vision network. Our weight-pruned CNN increases the speed by $3.5\times$ with a very small accuracy degradation to satisfy the requirement of information sharing for safe MARL behavior planning.}
\label{vision_result}
\begin{tabular}{c|cc}
\toprule
CNN Model                  & performance (\%) & speed (img/s) \\
\midrule
Lane Segmentation           & accuracy: 95.82 & 313 \\
3D object detection (OD)    & mAP: 76.5       & 28 \\
3D OD after weight pruning  & mAP: 73.6       & 99 \\
\bottomrule
\end{tabular}
\end{table}

\subsection{System Efficiency Improvement}

In this section, we show our algorithm improves the CAVs' system efficiency in terms of the average velocity and the average comfort as defined in the reward function~\eqref{equ:reward}. 

\subsubsection{Comparison Under Different CAV Ratios}

Our approach improves the average velocity and the average comfort as the CAV ratios (the total CAV number divided by the total number of all vehicles) get higher. In this set of experiments, the total number of CAVs ranges from 0 to 30 as listed in Table~\ref{tbl:mixed_setting}. We compare the average velocity and comfort for all vehicles under different CAV ratios. The comfort of a single vehicle is defined in Eq.~\eqref{equ:comfort}. It is averaged among all the vehicles as one criterion. The velocity and comfort are averaged over all the 40000 timesteps used in the simulation. The result of our approach is shown at the top of Table~\ref{tbl:mixed_setting}. All CAVs use our safe actor-critic Alg.~\ref{alg:actor-critic} introduced in Section~\ref{sec:learning_control}. The result in Table~\ref{tbl:mixed_setting} shows the average velocity and comfort of the entire mixed traffic. We use CARLA's built-in human-driven vehicle in the mixed traffic~\cite{Dosovitskiy17}. In the result of Table~\ref{tbl:mixed_setting}, the average velocity and comfort increase when the CAV ratio gets higher. This gives us insights that the penetration of the CAVs can improve traffic efficiency in the future. 

The truncated $\mathcal{Q}$-function technique introduced in Section~\ref{sec:info_sharing} is a good approximation of the centralized critic $Q(s,a)$ with the global $s$ and the global $a$. We compare our approach with the MADDPG~\cite{lowe2017multi} plus our safe action mapping technique introduced in Section~\ref{sec:learning_control}. The reason we add the safe action mapping is that the MADDPG algorithm stops early in each episode due to collisions. We will further explain this in safety experiments. The MADDPG uses the centralized critic $Q(s,a)$, so we use it as a baseline to evaluate our truncated $\mathcal{Q}$-function technique. Compared with the baseline, our average velocity is 6.0\% smaller and the average comfort is 3.9\% smaller. Though the system efficiency is sacrificed a little bit, our truncated $\mathcal{Q}$-function increases the scalability since it does not rely on the global states and actions.

\begin{table}[h]
  \centering
  \caption{The system efficiency comparison under different CAV ratios. Our approach improves the average velocity and the average comfort as the CAV ratios get higher.}
  \begin{tabular}{c|cc|cc}
  \toprule
  CAV & CAV  & human-driven & average  & average  \\
  ratio & number & vehicle number & velocity (mph) & comfort\\
  \midrule
  Algorithm & \multicolumn{4}{c}{Safe Actor-Critic (Truncated $\mathcal{Q}$ + Safe Action Mapping)} \\
  \midrule
  0 & 0 & 30 & 60.06 & 2.61 \\
  0.17 & 5 & 25 & 61.82 & 2.64 \\
  0.33 & 10 & 20 & 64.70 & 2.68 \\
  0.5 & 15 & 15 & 65.14 & 2.72 \\
  0.67 & 20 & 10 & 65.18 & 2.74 \\
  0.83 & 25 & 5 & 65.53 & 2.77 \\
  1 & 30 & 0 & 66.15 & 2.81 \\
  \midrule
  Algorithm & \multicolumn{4}{c}{MADDPG + Safe Action Mapping} \\
  \midrule
  0 & 0 & 30 & 60.06 & 2.61 \\
  0.17 & 5 & 25 & 62.80 & 2.65 \\
  0.33 & 10 & 20 & 66.19 & 2.72 \\
  0.5 & 15 & 15 & 67.71 & 2.83 \\
  0.67 & 20 & 10 & 67.80 & 2.83 \\
  0.83 & 25 & 5 & 69.31 & 2.85 \\
  1 & 30 & 0 & 70.34 & 2.89 \\
  \bottomrule
  \end{tabular}
  \label{tbl:mixed_setting}
\end{table}

\subsubsection{Comparison Under Different Traffic Densities}

Our approach improves the traffic flow and average comfort under different traffic densities. The traffic density $\rho$ is the ratio between the total number of vehicles and the road length. We compare the traffic flow and average comfort under different traffic densities between our safe MARL approach using Alg.~\ref{alg:actor-critic} and an intelligent driving model (IDM)~\cite{talebpour2016influence}. The traffic flow reflects the quality of the road throughout with respect to the traffic density. It is calculated as $\rho \times \bar{v},$ where $\bar{v}$ is the average velocity of all the vehicles \cite{rios2018impact}. The IDM is a common baseline in autonomous driving. In IDM, the vehicle's acceleration is a function of its current speed, current and desired spacing, and the leading and following vehicles' speed~\cite{talebpour2016influence}. Building on top of these IDM agents, we add lane-changing functionality using the gap acceptance method in~\cite{butakov2014personalized}. In this set of experiments, we keep CAV ratios at 0.6. As shown in Fig.~\ref{fig:idm}, the safe MARL agent gets both larger traffic flow and better driving comfort when traffic density $\rho$ is low. When $\rho$ grows, the result of the safe MARL agent gets worse, but it is still comparable with the IDM. When the road is saturated, lane-changing tends to downgrade passengers' comfort but cannot bring higher speed. Consequently, a better choice is to keep lanes when $\rho$ is high, and there is no significant difference between the safe MARL and the IDM. We also add the result using MADDPG in Fig.~\ref{fig:idm}. Both the traffic flow and the driving comfort are very small (a positive number but close to 0) using MADDPG, and many collisions occurred during the simulation. This is because the MADDPG does not have a safety guarantee.

\begin{figure}[h]
  \centering
  \includegraphics[width=3.4in]{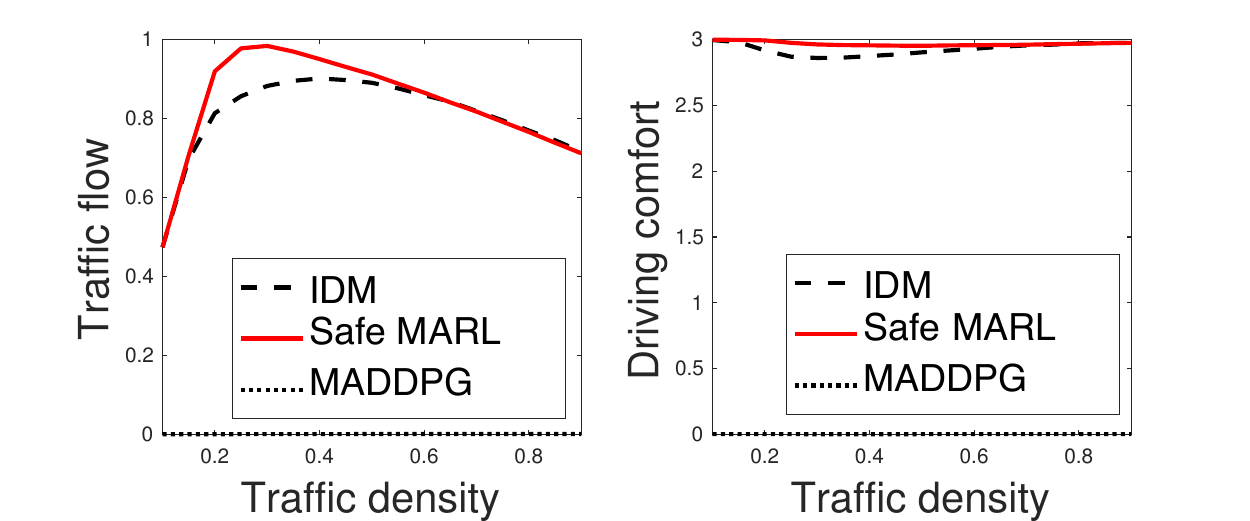}
  \caption{Our approach (safe MARL) improves the traffic flow and driving comfort under different traffic densities.}\label{fig:idm}
\end{figure}

\subsection{Safety Guarantee}

In this section, we show our algorithm has a safety guarantee. We compare our safe actor-critic algorithm with the MADDPG~\cite{lowe2017multi} algorithm. We assign a negative reward in MADDPG for each collision. Our approach avoids the execution of unsafe actions that can lead to collisions. Table~\ref{tbl:number_feedback} shows the total number of unsafe actions executed by the safe MARL and MADDPG under different traffic densities. The traffic density $\rho$ is the ratio between the total number of vehicles and the road length. The number in Table~\ref{tbl:number_feedback} is averaged over the last 10 episodes, which has a maximum timestep of 40000. When $\rho=0.9$, our approach has 0 unsafe actions while the MADDPG has 242976 unsafe actions because it does not have a safety module.

\begin{table}[h]
  \centering
  \caption{Total number of unsafe actions executed by the safe MARL and MADDPG under different traffic density $\rho$. Our approach has 0 unsafe action by using the safe action mapping.}
  \begin{tabular}{c|ccccc}
  \toprule
  $\rho$ & 0.1 & 0.2 & 0.3 & 0.4 & 0.5 \\
  \hline
  Safe MARL & 0 & 0 & 0 & 0 & 0 \\
  \hline
  MADDPG & 4416 & 9510 & 21897 & 43491 & 61689 \\
  \midrule
  $\rho$ & 0.6 & 0.7 & 0.8 & 0.9 & \\
  \hline
  Safe MARL & 0 & 0 & 0 & 0 &  \\
  \hline
  MADDPG & 91154 & 133135 & 191404 & 242976 & \\
  \bottomrule
  \end{tabular}
  \label{tbl:number_feedback}
\end{table}

Our approach can maintain a safe headway with neighboring vehicles while the MADDPG cannot. The headway is the distance between two consecutive vehicles following each other. In Fig.~\ref{fig:headway}, the minimum headway across all the vehicles is shown in the first 500 timesteps of one episode when $\rho = 0.6$. We see that the headway is always greater than 0 using our safe actor-critic algorithm with a minimum value of 18.5 meters. Nevertheless, the MADDPG is likely to have a negative headway, which means collisions in reality. Note that we set the minimum car-following distance of CAVs to be 18.5 meters following the study of the safe car-following distance of autonomous vehicles~\cite{arechiga2019specifying}, but this value can be set differently to satisfy the requirements in different scenarios.

\begin{figure}[h]
  \centering
  \includegraphics[width=3.4in]{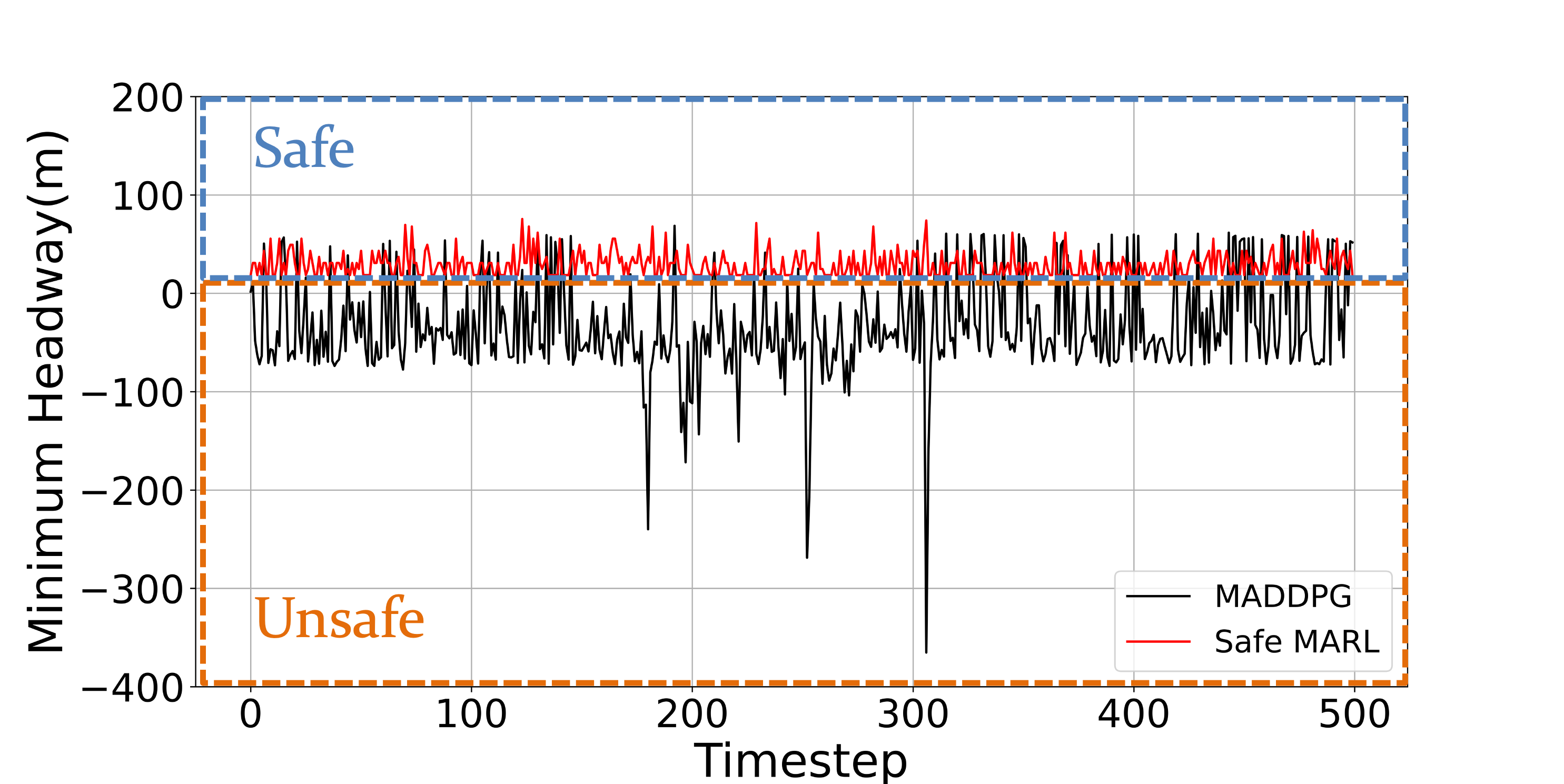}
  \vspace{-10pt}
  \caption{Our approach (Safe MARL) can maintain a safe headway with neighboring vehicles while the MADDPG cannot. The figure shows the minimum headway across all the vehicles in the first 500 timesteps of one episode when $\rho = 0.6$.}\label{fig:headway}
\end{figure}

Our approach gets a much larger total episode reward than the MADDPG. The total episode reward is the summation of all stage-wise rewards defined in Eq.~\ref{equ:reward} for each episode. As shown on the left of Fig.~\ref{fig:reward_result}, the maximum total episode reward using our approach is about 1940, which is firstly reached in the 20th episode. In the right figure, the maximum total episode reward using the MADDPG is about 7, because some collision terminates the episode. They have different initial rewards because the neural networks are randomly initialized and the action is selected by the $\epsilon$-greedy method with randomness. Our approach runs about 30 minutes in each episode as it has a safety guarantee and runs for the maximum episode length. Yet, each episode stops quickly in about 5 seconds using the MADDPG.

\begin{figure}
    \centering
    \subfigure{
    \begin{minipage}{1.6in}
    \centering
    \includegraphics[scale=0.14]{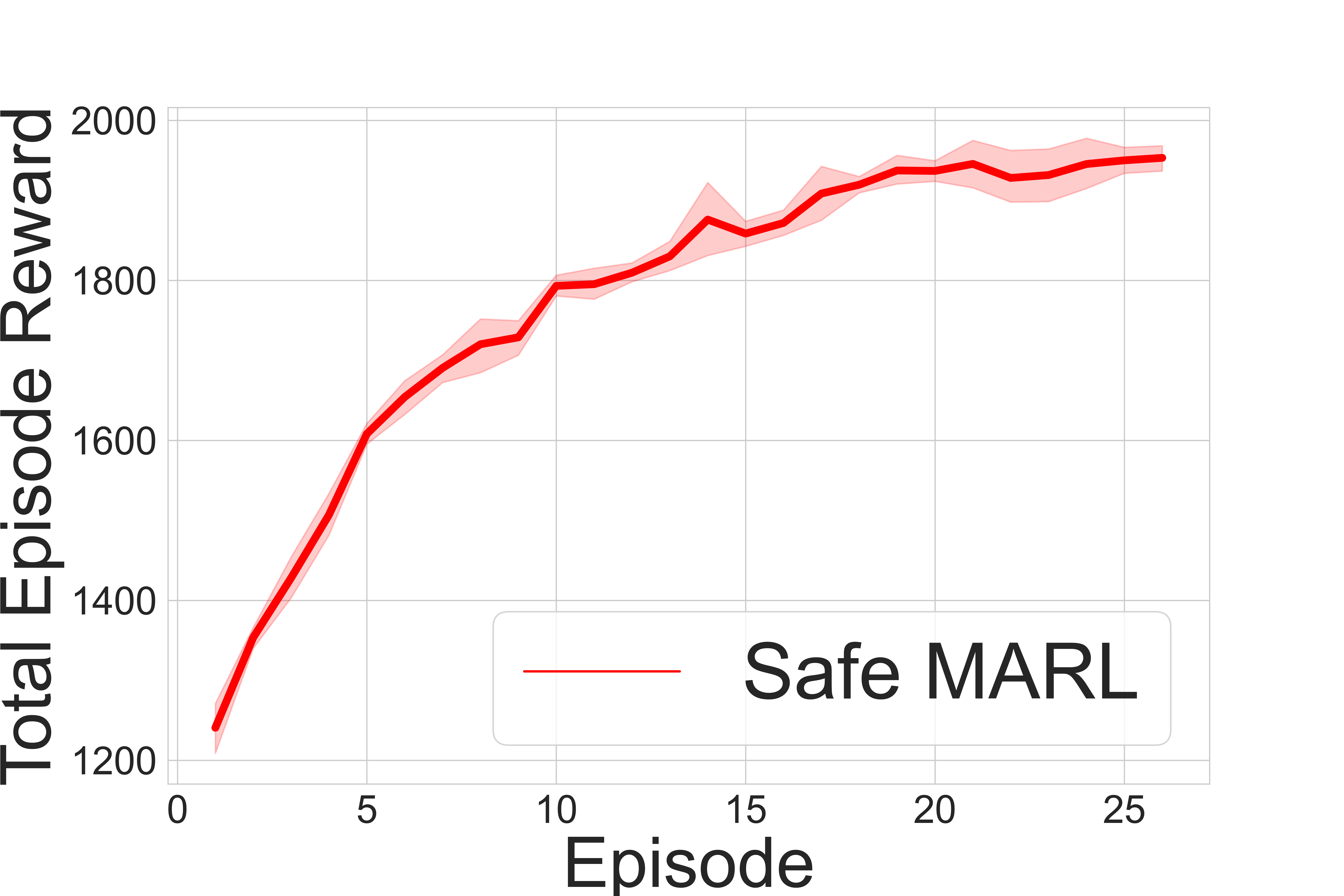}
    \end{minipage}
    }
    \subfigure{
    \begin{minipage}{1.6in}
    \centering
    \includegraphics[scale=0.14]{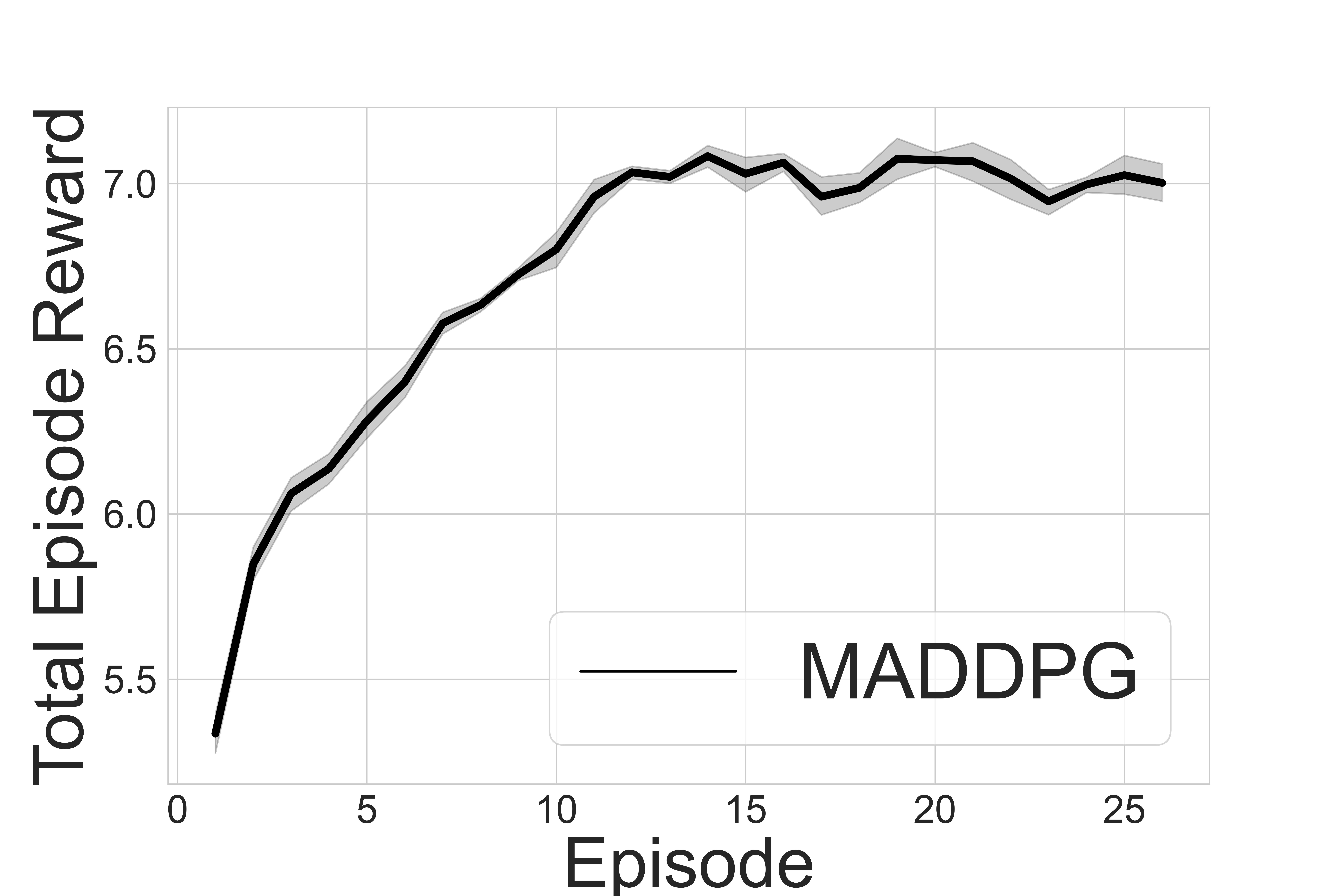}
    \end{minipage}
    }
    \caption{Our approach (safe MARL) gets a higher total episode reward compared to the MADDPG during the training process, because our approach can guarantee a safe training process and run for the maximum episode length.}\label{fig:reward_result}
\end{figure}

\subsection{Obstacle-At-Corner Scenario and Benefit of Shared Vision}

\begin{figure}[h]
  \centering
  \includegraphics[width=3in]{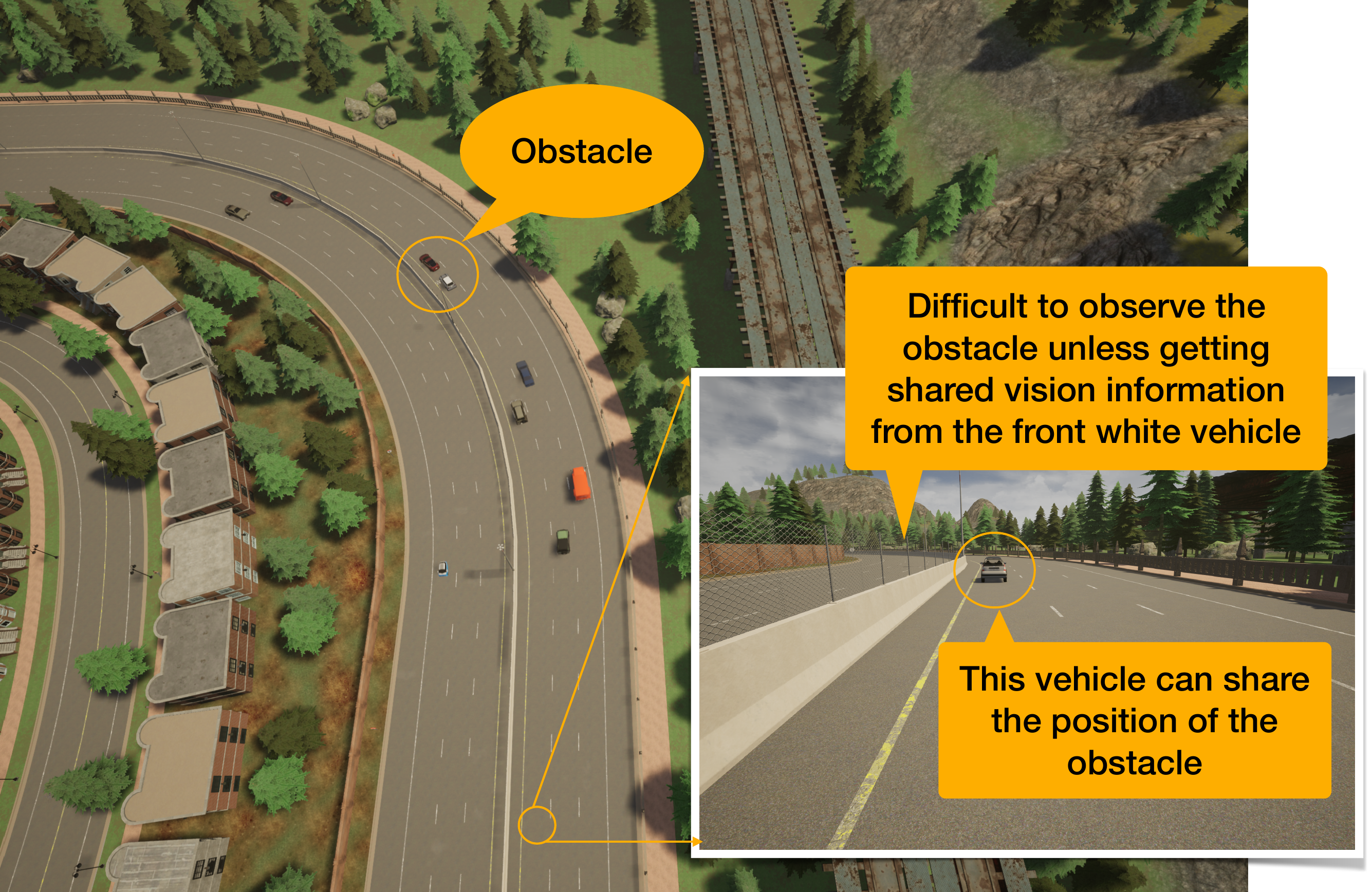}
  \caption{The obstacle-at-corner scenario where there are obstacles in a left-turning corner. The vehicles on the road are autonomous vehicles. The coming vehicles' view is blocked such that they cannot observe the obstacles.}\label{fig:obstacle}
\end{figure}

We construct a scenario called obstacle-at-corner to show how sharing vision information can help autonomous vehicles make wise lane-changing decisions ahead of time. As shown in Fig.~\ref{fig:obstacle}, there are obstacles at a left-turning corner (represented by two stationary vehicles). The right bottom figure shows the view of a vehicle that comes in the direction of this curve road. It is quite difficult to observe the obstacles merely relying on its own sensors. In this case, if the white front vehicle can share its observation, the coming vehicle can get to know there are obstacles before entering the turning corner.

\begin{figure}[h]
  \centering
  \includegraphics[width=1.8in]{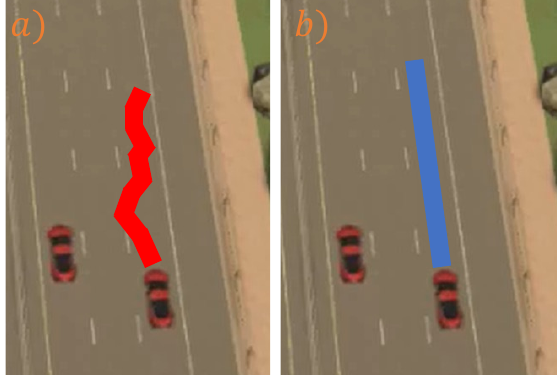}
  \vspace{-8pt}
  \caption{The trajectory is smooth when we separate the learning and control modules in (b), but the trajectory is zigzag when we use the steering angle and acceleration as the action space of our safe actor-critic algorithm in (a).}\label{fig:action_design}
\end{figure}

Separating the learning and control modules in our approach is necessary. One alternative design is to output the control inputs (steering angle and acceleration) directly~\cite{bojarski2016end,kendall2019learning,chen2020conditional,chen2021interpretable}. But it is not suitable for our CAV problem considering the lane-changing behavior. As the subfigure (a) shown in Fig.~\ref{fig:action_design}, we use the same truncated $\mathcal{Q}$-network structure in Fig.~\ref{fig:Qnet}, but the action space of both the critic and the actor networks is replaced by steering angle and acceleration. The trajectory is zigzag because the learned policy does not have a concept of changing/keeping lanes. The learned policy directly outputs the steering angles and accelerations to maximize the vehicle's rewards. Since the steering angle and acceleration are selected by the MARL agents, there is no need to use a controller in subfigure (a). In subfigure (b), the trajectory is smooth when we first find whether to change/keep lanes with the policy learned by Alg.~\ref{alg:actor-critic} and then implement this action by a CBF-QP controller in Eq.~\eqref{equ:QP_noise}.

\begin{figure}[h]
  \centering
  \includegraphics[width=3.2in]{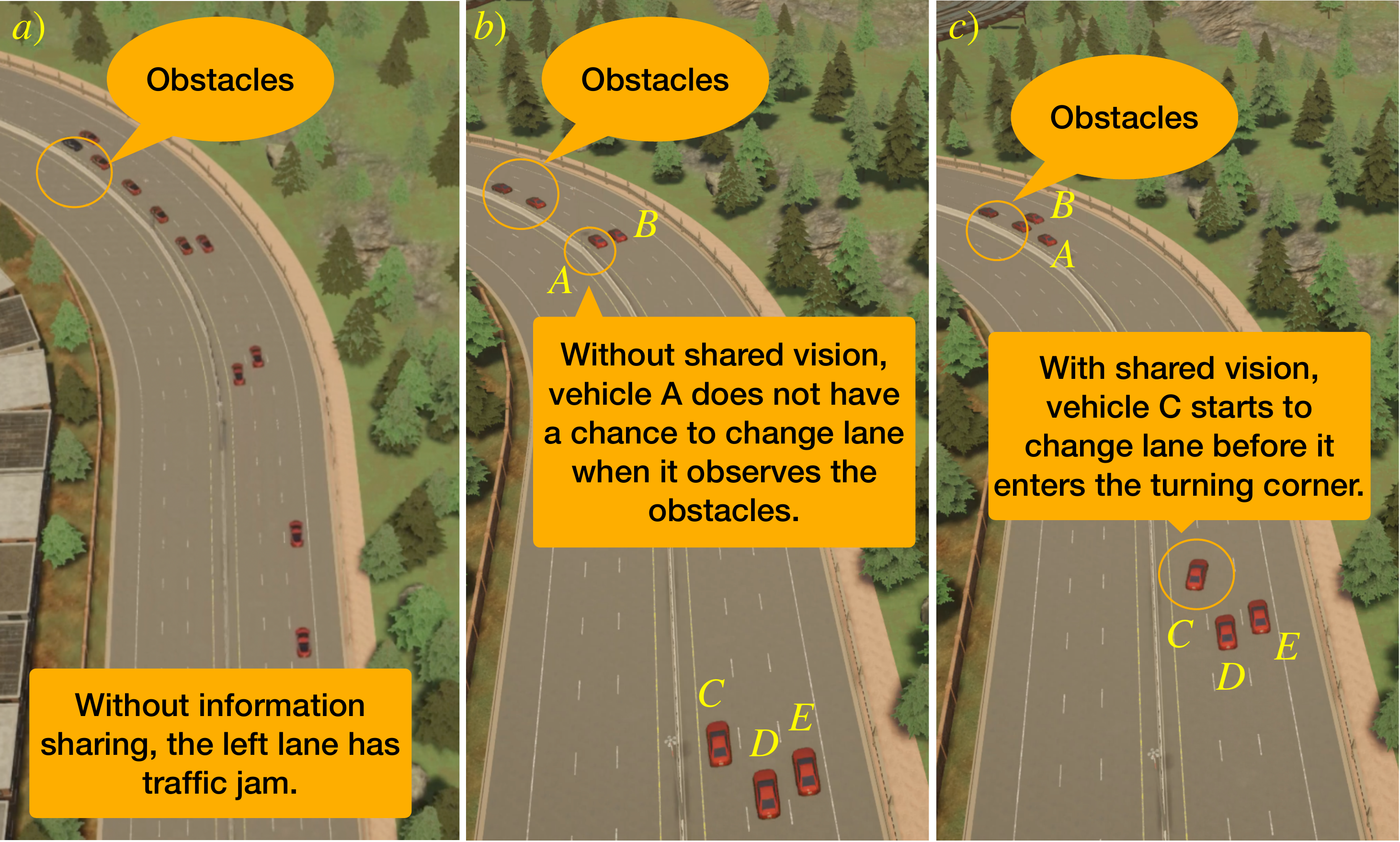}
  \caption{Without any information sharing, there is a traffic jam in the left lane. With shared vision from $A$ or $B$, using our safe MARL policy, vehicle $C$ can change its lane before it enters the left-turning corner.}\label{fig:lanechanging_corner}
\end{figure}

Shared vision with our proposed safe MARL algorithm can help CAVs avoid traffic jams. We test our learned policy in the obstacle-at-corner scenario. When no vehicle can share information in subfigure (a) of Fig.~\ref{fig:lanechanging_corner}, there are more vehicles blocked on the left-most lane and causing a traffic jam. Subfigures (b) and (c) are the screenshots taken in two close timesteps. In subfigure (b), vehicle $A$ is blocked in the left lane because there is vehicle $B$ in the middle lane and $A$ does not realize there are obstacles before it enters the corner (there is no neighboring vehicle that can share the obstacle information in advance in this scene;  two obstacles are not CAVs). In subfigure (c), the coming vehicle $C$ in the left-most lane starts to change to the middle lane before it observes the obstacles because it gets the shared vision information either from $A$ or $B$. We use the safe action mapping introduced in Sec.~\ref{sec:learning_control} to ensure the safety of this lane-changing.

\section{conclusion}
\label{sec:conclusion}
How to utilize V2V communication to improve traffic efficiency while satisfying safety requirements is a challenging problem for the behavior planning and control of CAVs. We design a safe behavior learning and control framework, which utilizes shared states and actions to safely explore a behavior planning policy.  We design two new techniques: truncated $\mathcal{Q}$-function and the safe action mapping. The truncated $\mathcal{Q}$-function utilizes the information-sharing capability of the CAVs instead of depending on the global states and actions, and we prove that the approximation error is bounded. The safe action mapping guarantees the safety of the training and execution process of the proposed MARL framework. In experiments, our weighted-pruned CNN technique increases the speed of the 3D object detection (OD) by 3.5$\times$ with small accuracy degradation to support the learning of the truncated $\mathcal{Q}$-function. We also show that our approach improves the average velocity and average comfort under different CAV ratios and different traffic densities. Our approach avoids unsafe actions and maintains a safe distance from neighboring vehicles. We also construct the obstacle-at-corner scenario to show that the shared vision can help vehicles to avoid traffic jams. It is considered as future work to improve the robustness of the MARL policy under state uncertainties that are caused by V2V communication or sensor measurement errors.

\appendices



\section{Physical Dynamic Model}
\label{sec:physical_model}
The physical dynamics of a vehicle are described by a kinematic bicycle model that achieves a good balance between accuracy and complexity \cite{brito2019model,rajamani2011vehicle}. The kinematic bicycle model is a widely-used model in literature~\cite{kim2021model,lindemann2021learning} for the MPC-based and CBF-based controller design in the CARLA simulator. The discrete-time equations are obtained by applying an explicit Euler method with a sampling time $T_s = 0.01s$:
\begin{equation}
x_{t+1} = \begin{matrix} \underbrace{ \begin{bmatrix}
x_1(t) + x_4(t) \cos(x_3(t)) T_s \\
x_2(t) + x_4(t) \sin(x_3(t)) T_s \\
x_3(t) \\
x_4(t) 
\end{bmatrix} } \\ f(x_t) \end{matrix}  + \begin{matrix} \underbrace{ \begin{bmatrix}
0 & 0 \\
0 & 0 \\
\frac{x_4(t) T_s}{L} & 0 \\
0 & T_s 
\end{bmatrix} } \\ g(x_t) \end{matrix}  u_t + w_t,
\label{equ:bicycle}
\end{equation}

where $x = [p_x, p_y, \psi, v]^\mathsf{T} \in \mathcal{X}$ includes the two coordinates of the center of gravity (CoG), its orientation and velocity. The input $u = [\tan(\delta), \dot{v}] \in \mathcal{U}$ is the tangent of the steering angle and acceleration. The parameter $L = 2.51 m$ is the distance between the front and rear axles.




\section{Exponential Decay Property}
\label{sec:exponential_decay}
As defined in Definition~\ref{def:exponential_decay}, the $(c,\rho)$-exponential decay property for $Q^i$-function means there exists some $c > 0$ and $ 0<\rho <1$ such that for any $i\in \mathcal{N}, \allowbreak \forall s^{N^i_\kappa} \in \mathcal{S}^{N^i_\kappa}, \allowbreak \forall s^{N^{-i}_\kappa} \in \mathcal{S}^{N^{-i}_\kappa}, \allowbreak \forall s'^{N^{-i}_\kappa} \in \mathcal{S}^{N^{-i}_\kappa}, \allowbreak \forall a^{N^i_\kappa} \in \mathcal{A}^{N^i_\kappa}, \allowbreak \forall a^{N^{-i}_\kappa} \in \mathcal{A}^{N^{-i}_\kappa}, \allowbreak \forall a'^{N^{-i}_\kappa} \in \mathcal{A}^{N^{-i}_\kappa}$, it holds that $|Q^i(s^{N^i_\kappa}, s^{N^{-i}_\kappa}, a^{N^i_\kappa}, a^{N^{-i}_\kappa})- Q^i(s^{N^i_\kappa}, s'^{N^{-i}_\kappa}, a^{N^i_\kappa}, a'^{N^{-i}_\kappa})| \le c \rho^{\kappa}$.

Here we prove the following theorem that gives the condition when the exponential decay property holds. 
\begin{theorem}
If for all $i \in \mathcal{N}$, $r^i$ is upper bounded by $\bar{r}$, then the $(\frac{\bar{r}}{1- \gamma}, \sqrt{\gamma})$-exponential decay property holds for $Q^i$ in \eqref{equ:q_function} under Assumption~\ref{aspt:independence}, where $\gamma$ is the discount factor.
\end{theorem}
\begin{proof}
Denote $s = (s^{N^i_\kappa}, s^{N^{-i}_\kappa})$, $a = (a^{N^i_\kappa}, a^{N^{-i}_\kappa})$, $s' = (s^{N^i_\kappa}, s'^{N^{-i}_\kappa})$, $a' = (a^{N^i_\kappa}, a'^{N^{-i}_\kappa})$. We use $p^i_k$ to represent the distribution of the state-action pair $(s^i_k, a^i_k)$ under some localized policy $\pi^i(a^i|s^{N^i})$ for vehicle $i$ with the initial condition $(s_0, a_0) = (s, a)$. Here the policy $\pi^i$ can be any feasible policy that is not required to be the optimal policy. We use $p_{k}'^i$ to represent the distribution of the state-action pair $(s^i_k, a^i_k)$ under the same policy $\pi^i$ for vehicle $i$ with the initial condition $(s_0, a_0) = (s', a')$. Because each vehicle $i$'s next state $s^i_{k+1}$ is independently generated and only depends on its 1-hop neighbors, that is to say,
\begin{equation}
    P(s_{k+1} | s_k, a_k) = \prod_{i=1}^{n} P(s^i_{k+1} | s^{N^i} _k, a^{N^i} _k),
\end{equation}
and the localized policy $\pi^i$ only depends on $s^{N^i}$, we have $p^i_k = p_{k}'^i$ for $k \leq \lfloor \kappa/2 \rfloor$. 

Based on the definition of the action-value function, we have 
\begin{align}
    & \left| Q^i(s, a) - Q^i(s', a') \right| \nonumber \\
    \leq & \sum_{k=0}^\infty \bigg| \Expect_{a_k^i \sim \pi(a^i |s^{N^i})} \left[ \gamma^k r^i_{k+1}(s^i_k, a^i_k) \mid s_o = s, a_0 = a\right] \nonumber \\
    & - \Expect_{a_k^i \sim \pi(a^i |s^{N^i})} \left[ \gamma^k r^i_{k+1}(s^i_k, a^i_k) \mid s_o = s', a_0 = a' \right] \bigg| \nonumber \\
    = & \sum_{k=0}^\infty \big| \gamma^k \Expect_{(s^i, a^i) \sim p^i_k} \left[r^i_{k+1}(s^i_k, a^i_k) \right] \nonumber \\
    & - \gamma^k \Expect_{(s^i, a^i) \sim p'^i_{k}} \left[r^i_{k+1}(s^i_k, a^i_k)\right] \big| \nonumber \\
    = & \sum_{k= \lfloor \kappa/2 \rfloor + 1}^\infty \left| \gamma^k \Expect_{p^i_k} \left[r^i_{k+1}(s^i_k, a^i_k)\right] - \gamma^k \Expect_{ p'^i_{k}} \left[r^i_{k+1}(s^i_k, a^i_k)\right] \right| \nonumber \\
    \leq & \sum_{k= \lfloor \kappa/2 \rfloor + 1}^\infty  \gamma^k \bar{r} \times \text{TV}(p^i_k, p'^i_{k}) \nonumber \\
    \leq & \frac{\bar{r}}{1- \gamma} \gamma^{\lfloor \kappa/2 \rfloor  + 1} \leq \frac{\bar{r}}{1- \gamma} \gamma^{\kappa/2}, 
\end{align}
where $\text{TV}(p^i_k, p'^i_{k})$ is the total variation distance between $p^i_k$ and $p'^i_{k}$ and upper bounded by $1$ based on the definition of total variation~\cite{grimmett2020probability}. 
This shows the the $(\frac{\bar{r}}{1- \gamma}, \sqrt{\gamma})$-exponential decay property holds for $Q^i$ according to the definition of the exponential decay property.
\end{proof}

\section{CBF-based Quadratic Programming}
\label{sec:CBF_QP}

In this section, we show how we can formulate a Quadratic Programming (QP) based on the CBF condition \eqref{equ:CBF} to guarantee the safety of a system without noise defined as follows:
\begin{equation}
    x_{t+1} = f(x_t) + g(x_t) u_t, 
    \label{equ:low_sys_equation}
\end{equation}
The CBF-QP for this system that can be solved efficiently at each timestep is formulated as:
\begin{equation}
\begin{aligned}
& \underset{u_t}{\argmin}
& & \norm{u_t - U \cdot a}^2 \\
& \text{s.t.}
& & p^\mathsf{T} f(x_t) + p^\mathsf{T} g(x_t) u_t + q \geq (1 - \eta) h(x_t).
\end{aligned}
\label{equ:QP_origin}
\end{equation}

\begin{lemma}
For system~\eqref{equ:low_sys_equation}, if there exists $\eta \in (0,1]$ such that the QP \eqref{equ:QP_origin} has a solution for all $x_t \in \mathcal{C}$ ($\mathcal{C}$ is defined in \eqref{equ:safe_set}), then the controller derived from \eqref{equ:QP_origin} renders set $\mathcal{C}$ forward invariant.
\label{lem:forward_invariant}
\end{lemma}
\begin{proof}
For any $x_t \in \mathcal{C}$, we have $h(x_t) \geq 0$ according to the Definition~\ref{def:safe_set}.
Therefore, 
\begin{equation}
\begin{aligned}
h(x_{t+1}) &= p^\mathsf{T} f(x_t) + p^\mathsf{T} g(x_t) u_t + q \\
&\geq (1 - \eta) h(x_t) \geq 0.
\end{aligned}
\end{equation}

Thus, $x_{t+1} \in \mathcal{C}$ and $\mathcal{C}$ is forward invariant.
\end{proof}

We can relax the constraint in \eqref{equ:QP_origin}, and the forward invariant property holds for a larger set. Consider:
\begin{equation}
\begin{aligned}
& \underset{u_t, \zeta}{\argmin} & & \norm{u_t - U \cdot a}^2 + M \zeta\\
& \text{s.t.} & & p^\mathsf{T} f(x_t) + p^\mathsf{T} g(x_t) u_t + q \geq (1 - \eta) h(x_t) - \zeta \\
& & & \zeta \geq 0,
\end{aligned}
\label{equ:QP_relaxed}
\end{equation}
where $M$ is a large constant.

\begin{theorem}
For system~\eqref{equ:low_sys_equation}, if there exists $\eta \in (0,1]$ such that the QP \eqref{equ:QP_relaxed} has a solution for all $x_t \in \mathcal{C}$ and the solution satisfies $\zeta \leq Z$, then the controller derived from \eqref{equ:QP_relaxed} renders set $\mathcal{C}': \{x \in \mathbb{R}^{n_x} \mid h'(x) = h(x) + \frac{Z}{\eta} \geq 0\}$ forward invariant.
\label{the:invariant_relaxed}
\end{theorem}
\begin{proof}
Since $x_t \in \mathcal{C}$, we have $h(x_t) \geq 0$ according to the Definition~\ref{def:safe_set}. 
From $Z \geq \zeta \geq 0 \geq - \eta h(x_t)$, we have $h(x_t) + \frac{Z}{\eta} \geq 0$. Thus, $x_t \in \mathcal{C}'$ with $h'(x_t) \geq 0$. Also, we have
\begin{equation}
\begin{aligned}
h(x_{t+1}) &= p^\mathsf{T} f(x_t) + p^\mathsf{T} g(x_t) u_t + q \\
& \geq (1 - \eta) h(x_t) - \zeta \geq (1 - \eta) h(x_t) - Z, \\
\end{aligned}
\end{equation}
\begin{equation}
\begin{aligned}
h(x_{t+1}) + \frac{Z}{\eta} & \geq (1 - \eta) \left[h(x_t) + \frac{Z}{\eta} \right], \\
h'(x_{t+1}) & \geq (1 - \eta) h'(x_t) \geq 0.
\end{aligned}
\end{equation}
Thus, $x_{t+1} \in \mathcal{C}'$ and $\mathcal{C}'$ is forward invariant.
Note that it is simplified to be Lemma~\ref{lem:forward_invariant} with set $\mathcal{C} = \mathcal{C}'$ when $Z = 0$.
\end{proof}

The value of $Z$ denotes how large the CBF condition \eqref{equ:CBF} is violated from the original $h(x_t)$. In this case, the safety condition should be formulated according to the set $\mathcal{C}'$. 


\bibliographystyle{IEEEtran}
\bibliography{IEEEabrv,connectedV}

\begin{thebibliography}{10}
\providecommand{\url}[1]{#1}
\csname url@rmstyle\endcsname
\providecommand{\newblock}{\relax}
\providecommand{\bibinfo}[2]{#2}
\providecommand\BIBentrySTDinterwordspacing{\spaceskip=0pt\relax}
\providecommand\BIBentryALTinterwordstretchfactor{4}
\providecommand\BIBentryALTinterwordspacing{\spaceskip=\fontdimen2\font plus
\BIBentryALTinterwordstretchfactor\fontdimen3\font minus
  \fontdimen4\font\relax}
\providecommand\BIBforeignlanguage[2]{{%
\expandafter\ifx\csname l@#1\endcsname\relax
\typeout{** WARNING: IEEEtran.bst: No hyphenation pattern has been}%
\typeout{** loaded for the language `#1'. Using the pattern for}%
\typeout{** the default language instead.}%
\else
\language=\csname l@#1\endcsname
\fi
#2}}

\bibitem{martin2020low}
D.~Mart{\'\i}n-Sacrist{\'a}n, S.~Roger, D.~Garcia-Roger, J.~F. Monserrat,
  P.~Spapis, C.~Zhou, and A.~Kaloxylos, ``Low-latency infrastructure-based
  cellular v2v communications for multi-operator environments with regional
  split,'' \emph{IEEE Trans. Intell. Transp. Syst.}, vol.~22, no.~2, pp.
  1052--1067, 2020.

\bibitem{mun2021secure}
H.~Mun, M.~Seo, and D.~H. Lee, ``Secure privacy-preserving v2v communication in
  5g-v2x supporting network slicing,'' \emph{IEEE Trans. Intell. Transp.
  Syst.}, 2021.

\bibitem{DSRC_standard}
J.~B. Kenney, ``Dedicated short-range communications (dsrc) standards in the
  united states,'' \emph{Proceedings of the IEEE}, vol.~99, no.~7, pp.
  1162--1182, July 2011.

\bibitem{won2021platooning}
M.~Won, ``L-platooning: A protocol for managing a long platoon with dsrc,''
  \emph{IEEE Trans. Intell. Transp. Syst.}, 2021.

\bibitem{Coordinate_CAV}
J.~{Rios-Torres} and A.~A. {Malikopoulos}, ``A survey on the coordination of
  connected and automated vehicles at intersections and merging at highway
  on-ramps,'' \emph{IEEE Trans. Intell. Transp. Syst.}, vol.~18, no.~5, pp.
  1066--1077, May 2017.

\bibitem{el2022novel}
Z.~e.~a. Kherroubi, S.~Aknine, and R.~Bacha, ``Novel decision-making strategy
  for connected and autonomous vehicles in highway on-ramp merging,''
  \emph{IEEE Trans. Intell. Transp. Syst.}, vol.~23, no.~8, pp.
  12\,490--12\,502, 2022.

\bibitem{farivar2021security}
F.~Farivar, M.~S. Haghighi, A.~Jolfaei, and S.~Wen, ``On the security of
  networked control systems in smart vehicle and its adaptive cruise control,''
  \emph{IEEE Trans. Intell. Transp. Syst.}, vol.~22, no.~6, pp. 3824--3831,
  2021.

\bibitem{silgu2021combined}
M.~A. Silgu, {\.I}.~G. Erda{\u{g}}i, G.~G{\"o}ksu, and H.~B. Celikoglu,
  ``Combined control of freeway traffic involving cooperative adaptive cruise
  controlled and human driven vehicles using feedback control through sumo,''
  \emph{IEEE Trans. Intell. Transp. Syst.}, 2021.

\bibitem{chen2020conditional}
L.~Chen, X.~Hu, B.~Tang, and Y.~Cheng, ``Conditional dqn-based motion planning
  with fuzzy logic for autonomous driving,'' \emph{IEEE Trans. Intell. Transp.
  Syst.}, 2020.

\bibitem{chen2021interpretable}
J.~Chen, S.~E. Li, and M.~Tomizuka, ``Interpretable end-to-end urban autonomous
  driving with latent deep reinforcement learning,'' \emph{IEEE Trans. Intell.
  Transp. Syst.}, 2021.

\bibitem{kiran2021deep}
B.~R. Kiran, I.~Sobh, V.~Talpaert, P.~Mannion, A.~A. Al~Sallab, S.~Yogamani,
  and P.~P{\'e}rez, ``Deep reinforcement learning for autonomous driving: A
  survey,'' \emph{IEEE Trans. Intell. Transp. Syst.}, vol.~23, no.~6, pp.
  4909--4926, 2022.

\bibitem{fu2022hybrid}
Y.~Fu, C.~Li, F.~R. Yu, T.~H. Luan, and Y.~Zhang, ``Hybrid autonomous driving
  guidance strategy combining deep reinforcement learning and expert system,''
  \emph{IEEE Trans. Intell. Transp. Syst.}, vol.~23, no.~8, pp.
  11\,273--11\,286, 2022.

\bibitem{zhu2022operational}
L.~Zhu, L.~Lu, X.~Wang, C.~Jiang, and N.~Ye, ``Operational characteristics of
  mixed-autonomy traffic flow on the freeway with on-and off-ramps and weaving
  sections: An rl-based approach,'' \emph{IEEE Trans. Intell. Transp. Syst.},
  vol.~23, no.~8, pp. 13\,512--13\,525, 2022.

\bibitem{lowe2017multi}
R.~Lowe and Y.~I. Wu, ``Multi-agent actor-critic for mixed
  cooperative-competitive environments,'' in \emph{NeurIPS}, 2017, pp.
  6379--6390.

\bibitem{rashid2018qmix}
T.~Rashid and M.~Samvelyan, ``Qmix: Monotonic value function factorisation for
  deep multi-agent reinforcement learning,'' in \emph{ICML}, 2018, pp.
  4295--4304.

\bibitem{muhammad2020deep}
K.~Muhammad, A.~Ullah, J.~Lloret, J.~Del~Ser, and V.~H.~C. de~Albuquerque,
  ``Deep learning for safe autonomous driving: Current challenges and future
  directions,'' \emph{IEEE Trans. Intell. Transp. Syst.}, vol.~22, no.~7, pp.
  4316--4336, 2021.

\bibitem{Dosovitskiy17}
A.~Dosovitskiy and G.~Ros, ``{CARLA}: {An} open urban driving simulator,'' in
  \emph{CoRL}, 2017, pp. 1--16.

\bibitem{le2022survey}
L.~Le~Mero, D.~Yi, M.~Dianati, and A.~Mouzakitis, ``A survey on imitation
  learning techniques for end-to-end autonomous vehicles,'' \emph{IEEE Trans.
  Intell. Transp. Syst.}, 2022.

\bibitem{cheng2019end}
R.~Cheng and G.~Orosz, ``End-to-end safe reinforcement learning through barrier
  functions for safety-critical continuous control tasks,'' in \emph{AAAI},
  vol.~33, 2019, pp. 3387--3395.

\bibitem{bojarski2016end}
M.~Bojarski and D.~Testa, ``End to end learning for self-driving cars,''
  \emph{arXiv:1604.07316}, 2016.

\bibitem{shalev2016safe}
S.~Shalev-Shwartz and S.~Shammah, ``Safe, multi-agent, reinforcement learning
  for autonomous driving,'' \emph{arXiv:1610.03295}, 2016.

\bibitem{aradi2020survey}
S.~Aradi, ``Survey of deep reinforcement learning for motion planning of
  autonomous vehicles,'' \emph{IEEE Trans. Intell. Transp. Syst.}, vol.~23,
  no.~2, pp. 740--759, 2022.

\bibitem{he2021rule}
S.~He, J.~Zeng, B.~Zhang, and K.~Sreenath, ``Rule-based safety-critical control
  design using control barrier functions with application to autonomous lane
  change,'' in \emph{ACC}.\hskip 1em plus 0.5em minus 0.4em\relax IEEE, 2021,
  pp. 178--185.

\bibitem{zhang2022unified}
T.~Zhang, W.~Song, M.~Fu, Y.~Yang, X.~Tian, and M.~Wang, ``A unified framework
  integrating decision making and trajectory planning based on spatio-temporal
  voxels for highway autonomous driving,'' \emph{IEEE Trans. Intell. Transp.
  Syst.}, vol.~23, no.~8, pp. 10\,365--10\,379, 2022.

\bibitem{zhang2019multi}
K.~Zhang and Z.~Yang, ``Multi-agent reinforcement learning: A selective
  overview of theories and algorithms,'' \emph{arXiv:1911.10635}, 2019.

\bibitem{iqbal2019actor}
S.~Iqbal and F.~Sha, ``Actor-attention-critic for multi-agent reinforcement
  learning,'' in \emph{ICML}, 2019, pp. 2961--2970.

\bibitem{sunehag2018value}
P.~Sunehag and G.~Lever, ``Value-decomposition networks for cooperative
  multi-agent learning based on team reward,'' in \emph{AAMAS}, 2018, pp.
  2085--2087.

\bibitem{yu2019distributed}
C.~Yu, X.~Wang, X.~Xu, M.~Zhang, H.~Ge, J.~Ren, L.~Sun, B.~Chen, and G.~Tan,
  ``Distributed multiagent coordinated learning for autonomous driving in
  highways based on dynamic coordination graphs,'' \emph{IEEE Trans. Intell.
  Transp. Syst.}, vol.~21, no.~2, pp. 735--748, 2019.

\bibitem{rashid_nips20}
T.~Rashid, G.~Farquhar, B.~Peng, and S.~Whiteson, ``Weighted qmix: Expanding
  monotonic value function factorisation for deep multi-agent reinforcement
  learning,'' in \emph{NeurIPS}, December 2020.

\bibitem{NEURIPS2020_8977ecbb}
M.~Zhou, Z.~Liu, P.~Sui, Y.~Li, and Y.~Y. Chung, ``Learning implicit credit
  assignment for cooperative multi-agent reinforcement learning,'' in
  \emph{NeurIPS}, vol.~33, 2020, pp. 11\,853--11\,864.

\bibitem{vinyals2019grandmaster}
O.~Vinyals and I.~Babuschkin, ``Grandmaster level in starcraft ii using
  multi-agent reinforcement learning,'' \emph{Nature}, vol. 575, no. 7782, pp.
  350--354, 2019.

\bibitem{qu2020scalable}
G.~Qu, A.~Wierman, and N.~Li, ``Scalable reinforcement learning of localized
  policies for multi-agent networked systems,'' in \emph{Learning for Dynamics
  and Control}.\hskip 1em plus 0.5em minus 0.4em\relax PMLR, 2020, pp.
  256--266.

\bibitem{bastani2018verifiable}
O.~Bastani, Y.~Pu, and A.~Solar-Lezama, ``Verifiable reinforcement learning via
  policy extraction,'' in \emph{NeurIPS}, 2018.

\bibitem{thomas2021safe}
G.~Thomas, Y.~Luo, and T.~Ma, ``Safe reinforcement learning by imagining the
  near future,'' \emph{NeurIPS}, vol.~34, pp. 13\,859--13\,869, 2021.

\bibitem{mo2022safe}
S.~Mo, X.~Pei, and C.~Wu, ``Safe reinforcement learning for autonomous vehicle
  using monte carlo tree search,'' \emph{IEEE Trans. Intell. Transp. Syst.},
  vol.~23, no.~7, pp. 6766--6773, 2022.

\bibitem{lu2021decentralized}
S.~Lu, K.~Zhang, T.~Chen, T.~Ba{\c{s}}ar, and L.~Horesh, ``Decentralized policy
  gradient descent ascent for safe multi-agent reinforcement learning,'' in
  \emph{AAAI}, vol.~35, no.~10, 2021, pp. 8767--8775.

\bibitem{li2020robust}
S.~Li and O.~Bastani, ``Robust model predictive shielding for safe
  reinforcement learning with stochastic dynamics,'' in \emph{ICRA}, 2020, pp.
  7166--7172.

\bibitem{zhang2019mamps}
W.~Zhang, O.~Bastani, and V.~Kumar, ``Mamps: Safe multi-agent reinforcement
  learning via model predictive shielding,'' \emph{arXiv:1910.12639}, 2019.

\bibitem{qin2021learning}
Z.~Qin, K.~Zhang, Y.~Chen, J.~Chen, and C.~Fan, ``Learning safe multi-agent
  control with decentralized neural barrier certificates,'' in \emph{ICLR},
  2021.

\bibitem{eskandarian2019research}
A.~Eskandarian, C.~Wu, and C.~Sun, ``Research advances and challenges of
  autonomous and connected ground vehicles,'' \emph{IEEE Trans. Intell. Transp.
  Syst.}, vol.~22, no.~2, pp. 683--711, 2019.

\bibitem{berlato2022smart}
S.~Berlato, M.~Centenaro, and S.~Ranise, ``Smart card-based identity management
  protocols for v2v and v2i communications in ccam: A systematic literature
  review,'' \emph{IEEE Trans. Intell. Transp. Syst.}, vol.~23, no.~8, pp.
  10\,086--10\,103, 2022.

\bibitem{notomista2020enhancing}
G.~Notomista, M.~Wang, M.~Schwager, and M.~Egerstedt, ``Enhancing
  game-theoretic autonomous car racing using control barrier functions,'' in
  \emph{ICRA}.\hskip 1em plus 0.5em minus 0.4em\relax IEEE, 2020, pp.
  5393--5399.

\bibitem{he2016deep}
K.~He, X.~Zhang, S.~Ren, and J.~Sun, ``Deep residual learning for image
  recognition,'' in \emph{CVPR}, 2016, pp. 770--778.

\bibitem{lang2019pointpillars}
A.~H. Lang, S.~Vora, H.~Caesar, L.~Zhou, J.~Yang, and O.~Beijbom,
  ``Pointpillars: Fast encoders for object detection from point clouds,'' in
  \emph{CVPR}, 2019, pp. 12\,697--12\,705.

\bibitem{miller2020cooperative}
A.~Miller and K.~Rim, ``Cooperative perception and localization for cooperative
  driving,'' in \emph{ICRA 2020}.\hskip 1em plus 0.5em minus 0.4em\relax IEEE,
  2020, pp. 1256--1262.

\bibitem{iandola2016squeezenet}
F.~N. Iandola and S.~Han, ``Squeezenet: Alexnet-level accuracy with 50x fewer
  parameters and $<$ 0.5 mb model size,'' \emph{arXiv:1602.07360}, 2016.

\bibitem{luo2017thinet}
J.-H. Luo, J.~Wu, and W.~Lin, ``Thinet: A filter level pruning method for deep
  neural network compression,'' in \emph{ICCV}, 2017, pp. 5058--5066.

\bibitem{ian2016deep}
I.~Goodfellow, Y.~Bengio, and A.~Courville, ``Deep learning (adaptive
  computation and machine learning series),'' \emph{Cambridge Massachusetts},
  pp. 321--359, 2017.

\bibitem{zhang2018systematic}
T.~Zhang, S.~Ye, K.~Zhang, J.~Tang, W.~Wen, M.~Fardad, and Y.~Wang, ``A
  systematic dnn weight pruning framework using alternating direction method of
  multipliers,'' in \emph{ECCV}, 2018, pp. 184--199.

\bibitem{khelladi2020emergency}
F.~Khelladi, M.~Boudali, R.~Orjuela, M.~Cassaro, M.~Basset, and C.~Roos, ``An
  emergency hierarchical guidance control strategy for autonomous vehicles,''
  \emph{IEEE Trans. Intell. Transp. Syst.}, 2020.

\bibitem{sutton2018reinforcement}
R.~S. Sutton and A.~G. Barto, \emph{Reinforcement learning: An
  introduction}.\hskip 1em plus 0.5em minus 0.4em\relax MIT press, 2018.

\bibitem{agrawal2017discrete}
A.~Agrawal and K.~Sreenath, ``Discrete control barrier functions for
  safety-critical control of discrete systems with application to bipedal robot
  navigation.'' in \emph{RSS}, vol.~13.\hskip 1em plus 0.5em minus 0.4em\relax
  Cambridge, MA, USA, 2017.

\bibitem{safeRL_aaai19}
R.~Cheng, G.~Orosz, R.~M. Murray, and J.~W. Burdick, ``End-to-end safe
  reinforcement learning through barrier functions for safety-critical
  continuous control tasks,'' in \emph{AAAI}, vol.~33, 2019, pp. 3387--3395.

\bibitem{ames2019control}
A.~D. Ames and S.~Coogan, ``Control barrier functions: Theory and
  applications,'' in \emph{ECC 2019}.\hskip 1em plus 0.5em minus 0.4em\relax
  IEEE, 2019, pp. 3420--3431.

\bibitem{cesari2017scenario}
G.~Cesari and G.~Schildbach, ``Scenario model predictive control for lane
  change assistance and autonomous driving on highways,'' \emph{IEEE Intell.
  Transp. Syst. Mag.}, vol.~9, no.~3, pp. 23--35, 2017.

\bibitem{li2022autonomous}
B.~Li, Y.~Ouyang, L.~Li, and Y.~Zhang, ``Autonomous driving on curvy roads
  without reliance on frenet frame: A cartesian-based trajectory planning
  method,'' \emph{IEEE Trans. Intell. Transp. Syst.}, 2022.

\bibitem{dixit2018trajectory}
S.~Dixit, S.~Fallah, U.~Montanaro, M.~Dianati, A.~Stevens, F.~Mccullough, and
  A.~Mouzakitis, ``Trajectory planning and tracking for autonomous overtaking:
  State-of-the-art and future prospects,'' \emph{Annu. Rev. Control}, vol.~45,
  pp. 76--86, 2018.

\bibitem{qin2020ultra}
Z.~Qin, H.~Wang, and X.~Li, ``Ultra fast structure-aware deep lane detection,''
  \emph{arXiv:2004.11757}, 2020.

\bibitem{net_summary}
A.~Benali~Amjoud and M.~Amrouch, ``Convolutional neural networks backbones for
  object detection,'' in \emph{Image and Signal Process.}\hskip 1em plus 0.5em
  minus 0.4em\relax Springer, 2020.

\bibitem{xie2017aggregated}
S.~Xie and R.~Girshick, ``Aggregated residual transformations for deep neural
  networks,'' in \emph{CVPR}, 2017, pp. 1492--1500.

\bibitem{krizhevsky2012imagenet}
A.~Krizhevsky and I.~Sutskever, ``Imagenet classification with deep
  convolutional neural networks,'' \emph{NeurIPS}, vol.~25, pp. 1097--1105,
  2012.

\bibitem{simonyan2014very}
K.~Simonyan and A.~Zisserman, ``Very deep convolutional networks for
  large-scale image recognition,'' \emph{arXiv: 1409.1556}, 2014.

\bibitem{szegedy2015going}
C.~Szegedy, W.~Liu, Y.~Jia, P.~Sermanet, S.~Reed, D.~Anguelov, D.~Erhan,
  V.~Vanhoucke, and A.~Rabinovich, ``Going deeper with convolutions,'' in
  \emph{CVPR}, 2015, pp. 1--9.

\bibitem{qi2017pointnet}
C.~R. Qi, H.~Su, K.~Mo, and L.~J. Guibas, ``Pointnet: Deep learning on point
  sets for 3d classification and segmentation,'' in \emph{CVPR}, 2017, pp.
  652--660.

\bibitem{zhou2018voxelnet}
Y.~Zhou and O.~Tuzel, ``Voxelnet: End-to-end learning for point cloud based 3d
  object detection,'' in \emph{CVPR}, 2018, pp. 4490--4499.

\bibitem{liu2016ssd}
W.~Liu, D.~Anguelov, D.~Erhan, C.~Szegedy, S.~Reed, C.-Y. Fu, and A.~C. Berg,
  ``Ssd: Single shot multibox detector,'' in \emph{ECCV}.\hskip 1em plus 0.5em
  minus 0.4em\relax Springer, 2016, pp. 21--37.

\bibitem{talebpour2016influence}
A.~Talebpour and H.~S. Mahmassani, ``Influence of connected and autonomous
  vehicles on traffic flow stability and throughput,'' \emph{Transp. Res. Part
  C Emerg. Technol.}, vol.~71, pp. 143--163, 2016.

\bibitem{rios2018impact}
J.~Rios-Torres and A.~A. Malikopoulos, ``Impact of partial penetrations of
  connected and automated vehicles on fuel consumption and traffic flow,''
  \emph{IEEE Trans. Intell. Veh}, vol.~3, no.~4, pp. 453--462, 2018.

\bibitem{butakov2014personalized}
V.~A. Butakov and P.~Ioannou, ``Personalized driver/vehicle lane change models
  for adas,'' \emph{IEEE Trans. Veh. Technol.}, vol.~64, no.~10, pp.
  4422--4431, 2014.

\bibitem{arechiga2019specifying}
N.~Arechiga, ``Specifying safety of autonomous vehicles in signal temporal
  logic,'' in \emph{2019 IEEE Intelligent Vehicles Symposium (IV)}.\hskip 1em
  plus 0.5em minus 0.4em\relax IEEE, 2019, pp. 58--63.

\bibitem{kendall2019learning}
A.~Kendall, J.~Hawke, D.~Janz, P.~Mazur, D.~Reda, J.-M. Allen, V.-D. Lam,
  A.~Bewley, and A.~Shah, ``Learning to drive in a day,'' in \emph{2019
  ICRA}.\hskip 1em plus 0.5em minus 0.4em\relax IEEE, 2019, pp. 8248--8254.

\bibitem{brito2019model}
B.~Brito and B.~Floor, ``Model predictive contouring control for collision
  avoidance in unstructured dynamic environments,'' \emph{IEEE Robot. Autom.
  Lett}, vol.~4, no.~4, pp. 4459--4466, 2019.

\bibitem{rajamani2011vehicle}
R.~Rajamani, \emph{Vehicle dynamics and control}.\hskip 1em plus 0.5em minus
  0.4em\relax Springer Science \& Business Media, 2011.

\bibitem{kim2021model}
M.~Kim, D.~Lee, J.~Ahn, M.~Kim, and J.~Park, ``Model predictive control method
  for autonomous vehicles using time-varying and non-uniformly spaced
  horizon,'' \emph{IEEE Access}, vol.~9, pp. 86\,475--86\,487, 2021.

\bibitem{lindemann2021learning}
L.~Lindemann, A.~Robey, L.~Jiang, S.~Tu, and N.~Matni, ``Learning robust output
  control barrier functions from safe expert demonstrations,'' \emph{arXiv
  preprint arXiv:2111.09971}, 2021.

\bibitem{grimmett2020probability}
G.~Grimmett and D.~Stirzaker, \emph{Probability and random processes}.\hskip
  1em plus 0.5em minus 0.4em\relax Oxford university press, 2020.

\end{thebibliography}
\vspace{-50pt}
\begin{IEEEbiography}[{\includegraphics[width=1in,height=1.25in,clip,keepaspectratio]{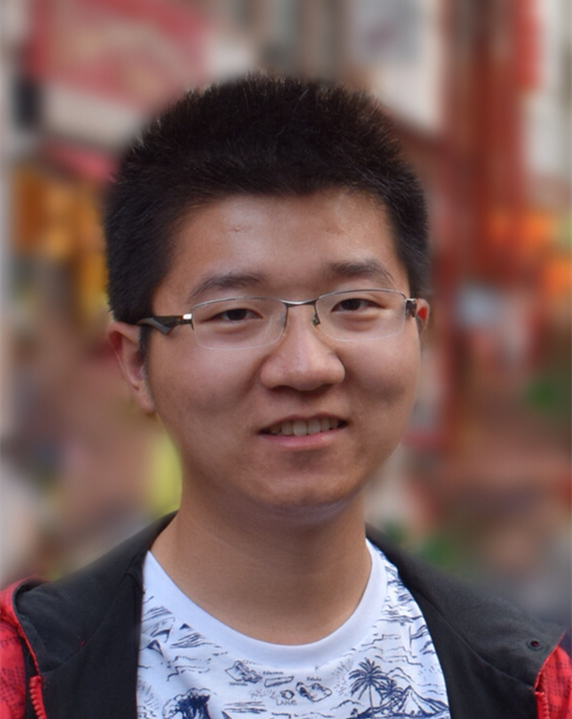}}]{Songyang Han}(S'17) received the B.E. degree in automation from Nanjing University, Nanjing, China, in 2015, and the M.S. degree in electrical and computer engineering from the University of Michigan-Shanghai Jiao Tong University Joint Institute, Shanghai Jiao Tong University, Shanghai, China, in 2018. He is currently working toward the Ph.D. degree in computer science and engineering at the University of Connecticut, Storrs, CT, USA. His current research interests include reinforcement learning, artificial intelligence, deep learning, autonomous driving, computer vision, and game theory. He is the recipient of the Best Paper Award at the 12th ACM/IEEE International Conference on Cyber-Physical Systems.
\end{IEEEbiography}
\vskip -1\baselineskip plus -1fil

\begin{IEEEbiography}[{\includegraphics[width=1in,height=1.25in,clip,keepaspectratio]{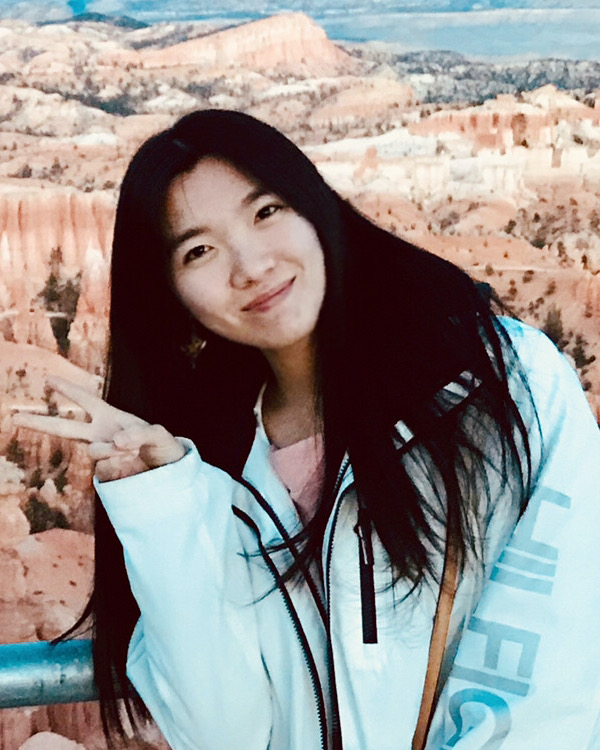}}]{Shanglin Zhou} received the B.S. degree in statistics from Minzu University of China, Beijing, China, in 2015, and the M.S. degree in statistics from the University of Connecticut, Storrs, CT, USA, in 2017. She is currently working toward the Ph.D. degree in computer science and engineering at the University of Connecticut, Storrs, CT, USA. Her current research interests include deep learning model compression and the application of computer vision.
\end{IEEEbiography}
\vskip -1\baselineskip plus -1fil

\begin{IEEEbiography}[{\includegraphics[width=1in,height=1.25in,clip,keepaspectratio]{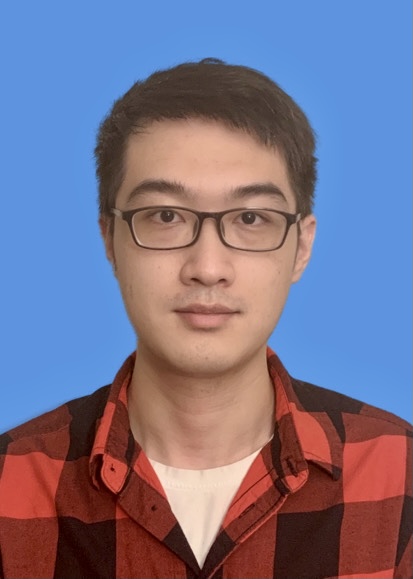}}]{Jiangwei Wang}(S'19) received the B.S. degree in electrical and computer engineering from Xi'an Jiaotong University, Xi'an, China, in 2018. He is currently working toward the Ph.D. degree in electrical and computer engineering at the University of Connecticut, Storrs, CT, USA. His current research interests are in cyber-physical systems and autonomous driving. He was also a recipient of Outstanding Reviewer for the IEEE Transactions on Sustainable Energy and IEEE Journal of Oceanic Engineering in 2019.
\end{IEEEbiography}
\vskip -1\baselineskip plus -1fil

\begin{IEEEbiography}[{\includegraphics[width=1in,height=1.25in,clip,keepaspectratio]{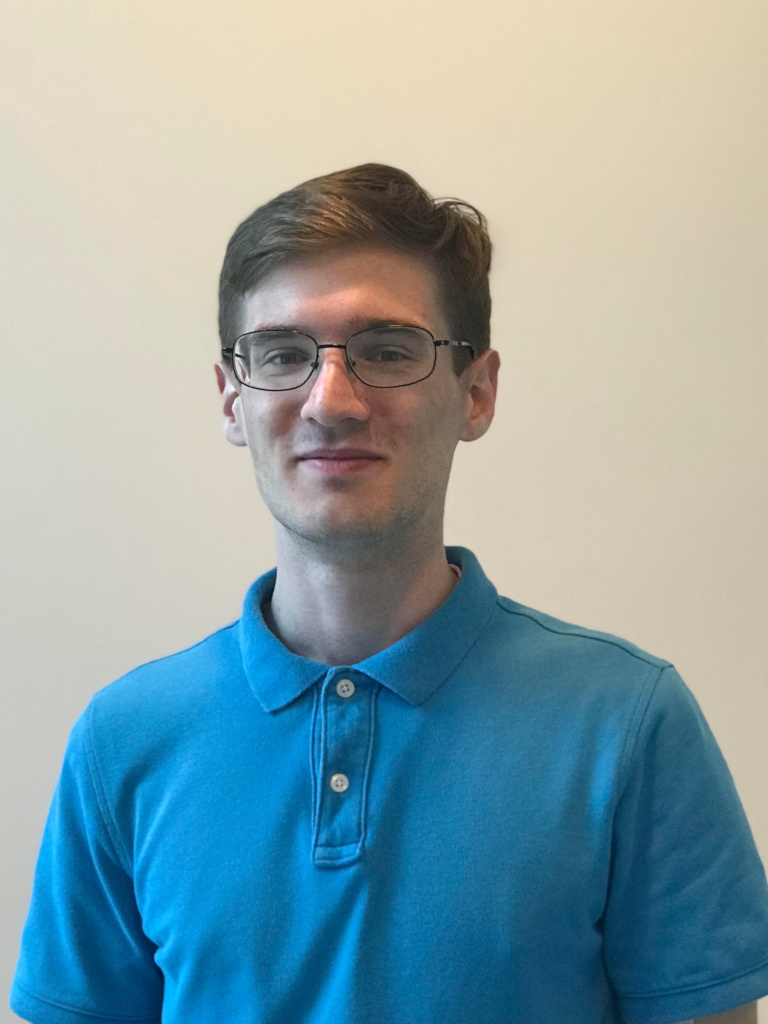}}]{Lynn Pepin} received the B.E. degree in computer science and engineering from the University of Connecticut, Storrs, CT, USA, in 2018. He is currently working toward the Ph.D. degree in computer science at the University of Connecticut, Storrs, CT, USA. His current research interests include cybersecurity and machine learning research in the area of power systems and autonomous vehicles.
\end{IEEEbiography}
\vskip -1\baselineskip plus -1fil

\begin{IEEEbiography}[{\includegraphics[width=1in,height=1.25in,clip,keepaspectratio]{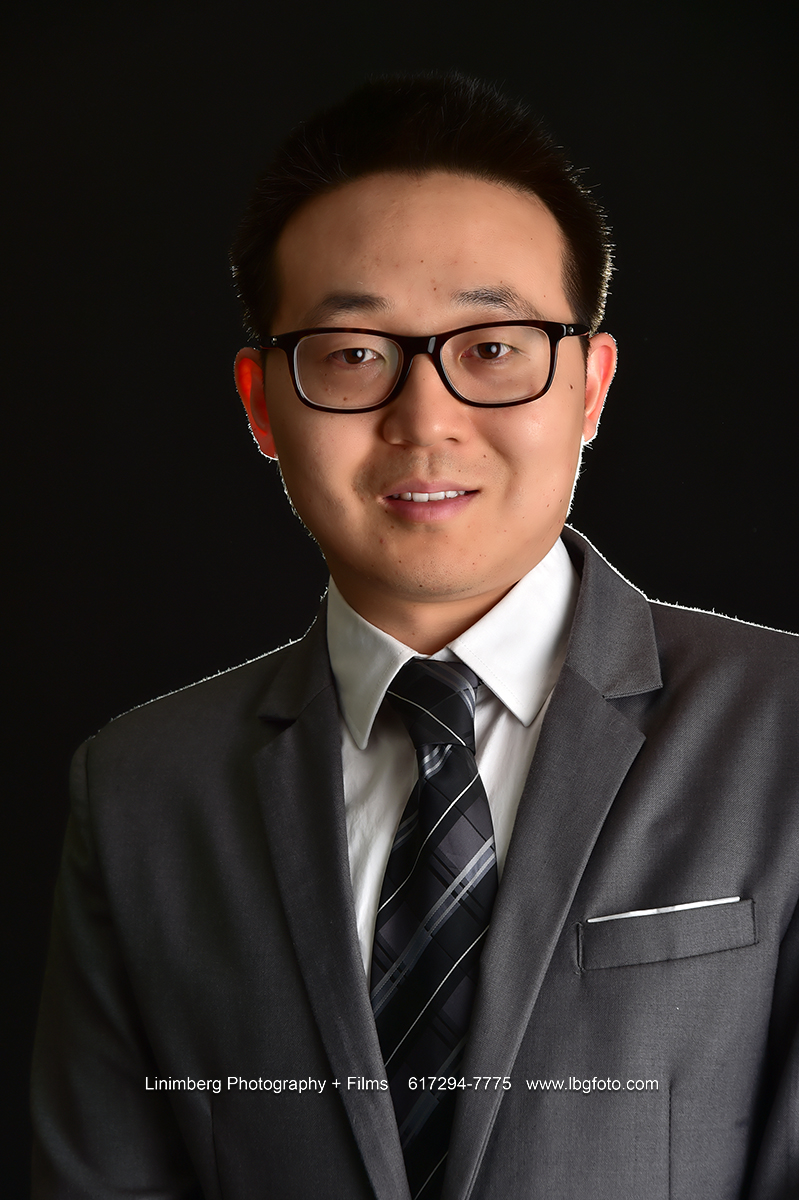}}]{Caiwen Ding} received the Ph.D. degree in computer science and engineering from the Northeastern University, Boston, MA, USA, in 2019. He is currently an Assistant Professor in the Department of Computer Science and Engineering at the University of Connecticut, Storrs, CT, USA. His research interests include machine learning and deep neural network systems, computer vision, and natural language processing. Dr. Ding is the recipient of the Best Paper Award Nomination at DATE 2018 and DATE 2021.
\end{IEEEbiography}
\vskip -1\baselineskip plus -1fil

\begin{IEEEbiography}[{\includegraphics[width=1in,height=1.25in,clip,keepaspectratio]{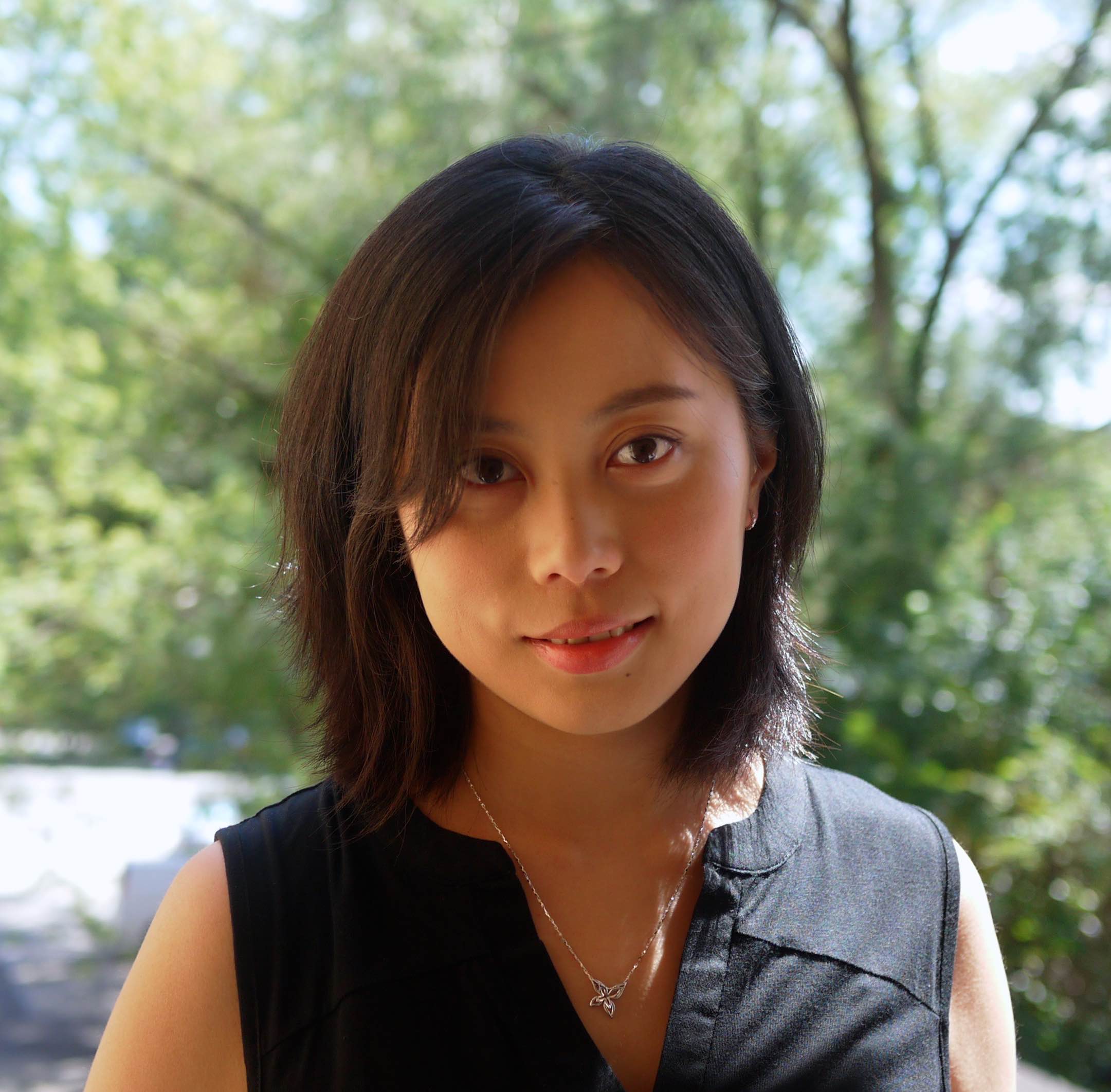}}]{Jie Fu} received the M.S. degree in Electrical Engineering and Automation from Beijing Institute of Technology, Beijing, China, in 2009, and the Ph.D. degree in Mechanical Engineering from the University of Delaware, Newark, DE, USA, in 2013. She is currently an Assistant Professor in the Department of Electrical and Computer Engineering at University of Florida, Gainesville, FL, USA 32605. Previously, she was a Postdoc Researcher in the Department of Electrical and Systems Engineering at the University of Pennsylvania. Her research interests include control theory, formal methods, and machine learning, with applications to robotic systems and cyber-physical systems.
\end{IEEEbiography}
\vskip -1\baselineskip plus -1fil

\begin{IEEEbiography}[{\includegraphics[width=1in,height=1.25in,clip,keepaspectratio]{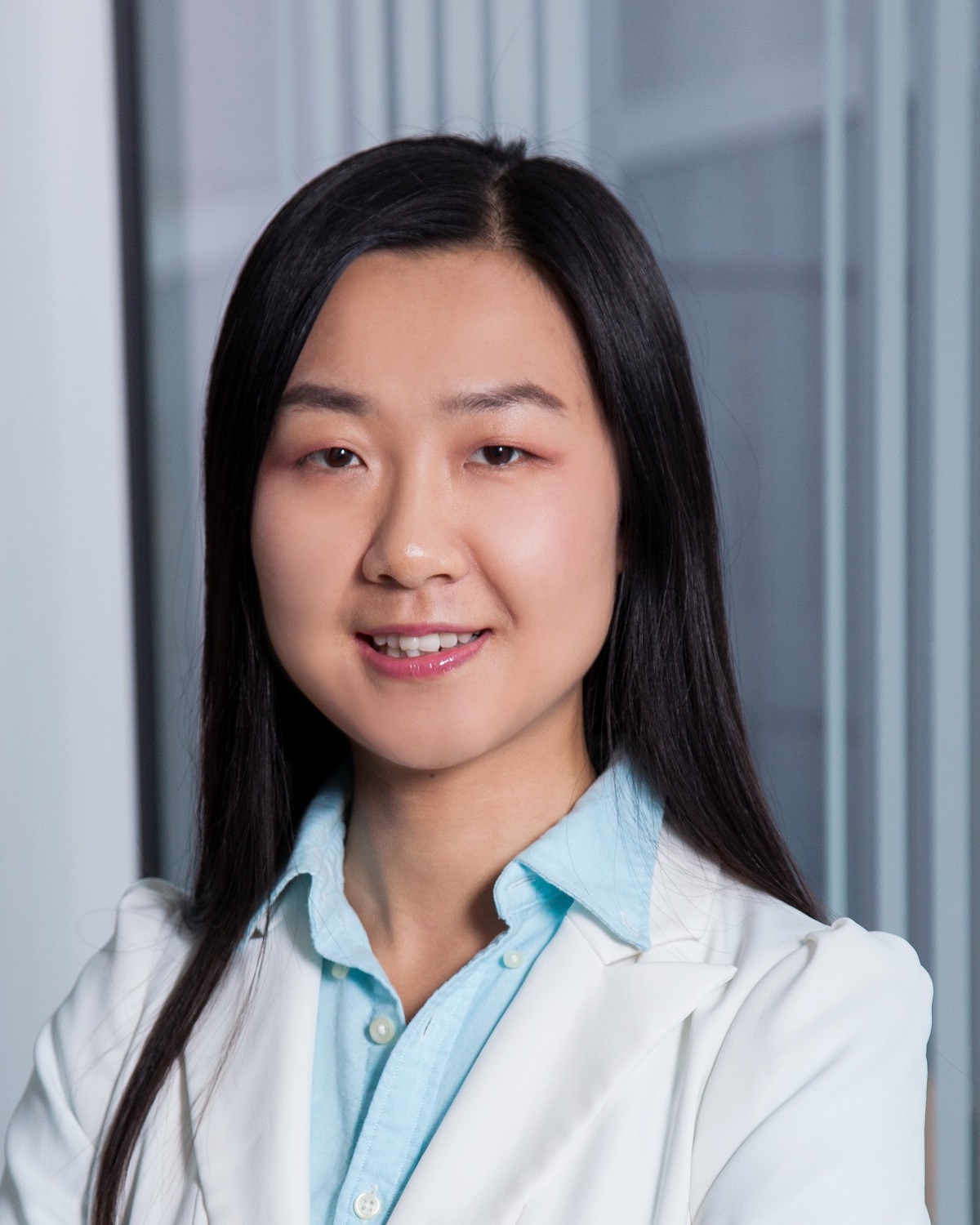}}]{Fei Miao} (S'13-M'16) received the B.S. degree in automation from Shanghai Jiao Tong University, Shanghai, China, in 2010, and the M.A. degree in statistics and the Ph.D. degree with the ``Charles Hallac and Sarah Keil Wolf Award for Best Doctoral Dissertation" in electrical and systems engineering both from the University of Pennsylvania, Philadelphia, PA, USA, in 2015 and 2016. She is currently an Assistant Professor in the Department of Computer Science and Engineering at the University of Connecticut, Storrs, CT, USA. Previously, she was a Postdoc Researcher in the Department of Electrical and Systems Engineering at the University of Pennsylvania. Her research interests include learning and control of cyber-physical systems under model uncertainties, and CPS security. Dr. Miao is the recipient of the NSF CAREER Award, and Best Paper Award Finalist at the 6th and 12th ACM/IEEE International Conference on Cyber-Physical Systems.
\end{IEEEbiography}

\end{document}